\documentclass{article}

\usepackage{log_2026}
\usepackage[numbers,compress,sort]{natbib}

\usepackage[utf8]{inputenc}
\usepackage[T1]{fontenc}
\usepackage{graphicx}
\usepackage{subcaption}
\usepackage{booktabs}
\usepackage{url}
\usepackage{nicefrac}
\usepackage{microtype}

\usepackage{inconsolata}

\usepackage{mathtools}
\usepackage{amsthm}
\usepackage{array}     
\newcolumntype{P}[1]{>{\centering\arraybackslash}p{#1}}

\usepackage{multirow} 
\usepackage{tabularx} 

\usepackage{amsmath,amssymb,amsfonts}

\usepackage{algorithm}
\usepackage{algorithmic}  
\usepackage[dvipsnames,table]{xcolor}
\usepackage{soul}

\definecolor{blue}{RGB}{0,0,255}
\definecolor{teal}{RGB}{0,128,128}
\definecolor{violet}{RGB}{138,43,226}
\definecolor{lightblue}{rgb}{0.9, 0.95, 1.0}
\definecolor{mygreen}{rgb}{0.01, 0.5, 0.01}
\definecolor{myred}{rgb}{0.8, 0.01, 0.01}

\usepackage[
  backref=page,
  colorlinks,
  urlcolor=purple,
  linkcolor=purple,
  citecolor=purple
]{hyperref}

\usepackage{appendix}
\usepackage{enumitem}
\usepackage{bm}

\usepackage[capitalize,noabbrev]{cleveref}

\theoremstyle{plain}
\newtheorem{theorem}{Theorem}[section]
\newtheorem{proposition}[theorem]{Proposition}

\theoremstyle{definition}

\theoremstyle{remark}
\newtheorem{remark}[theorem]{Remark}

\usepackage[disable,textsize=tiny]{todonotes}

\title[ChronoSpike]{ChronoSpike: An Adaptive Spiking Graph Neural Network for Dynamic Graphs}

\author{
  \textbf{Md Abrar Jahin}$^\diamondsuit$ \quad
  \textbf{Taufikur Rahman Fuad}$^\dagger$ \quad
  \textbf{Jay Pujara}$^\diamondsuit$ \quad
  \textbf{Craig A. Knoblock}$^\diamondsuit$ \\
  $^\diamondsuit$University of Southern California\\
  $^\dagger$Islamic University of Technology
}

\begin{document}

\maketitle

\begin{abstract}
  Dynamic graph representation learning requires capturing both structural relations and temporal evolution, yet existing approaches face a core trade-off: attention-based methods offer expressiveness at $O(T^2)$ complexity, while recurrent architectures suffer from gradient pathologies and dense state storage. Spiking neural networks provide event-driven efficiency but are constrained by sequential propagation, binary information loss, and local aggregation that lacks global context. We propose \textbf{ChronoSpike}, an adaptive spiking graph neural network that integrates learnable LIF neurons with per-channel membrane dynamics, multi-head spatially-attentive aggregation over continuous features, and a lightweight Transformer temporal encoder. This design enables fine-grained local modeling and long-range dependency capture with $O(T \cdot d)$ activation/state memory and an additional $O(T^2)$ per-node attention term that remains small for the horizons evaluated here. ChronoSpike outperforms twelve state-of-the-art baselines on three large benchmarks by 2.0\%~Macro-F1 and 2.4\%~Micro-F1 on average while achieving $3\text{--}10\times$ faster training than recurrent methods with a constant 105K-parameter budget independent of graph size. We provide theoretical guarantees for membrane potential boundedness, gradient flow stability under contraction factor $\rho\!<\!1$, and BIBO stability; interpretability analyses reveal heterogeneous temporal receptive fields and a learned primacy effect with 83--88\% sparsity. Our code is publicly available at: \url{https://github.com/Abrar2652/ChronoSpike}.
\end{abstract}

\section{Introduction}
\label{sec:intro}
Real-world systems, such as financial transaction networks, social interactions, and biological processes, are inherently dynamic, with structures and node attributes evolving over time. Capturing both structural dependencies and temporal evolution is the central challenge of \emph{Dynamic Graph Representation Learning} (DGRL) \cite{dynamic_graph_survey, yang2023dynamic}. The goal of DGRL is to learn low-dimensional node representations that encode historical dynamics and support downstream tasks such as temporal node classification, link prediction, and anomaly detection \cite{kumar2019predicting_jodie, tgat}. While Graph Neural Networks (GNNs) have achieved strong performance on static graphs \cite{kipf2016semi}, extending them to large-scale dynamic settings efficiently and expressively remains an open problem.

Existing approaches to DGRL fall into several paradigms, each with fundamental scalability or expressiveness limits. \emph{Recurrent architectures} (e.g., JODIE \cite{kumar2019predicting_jodie}, EvolveGCN \cite{pareja2020evolvegcn}) model temporal evolution with recurrent neural networks, but must maintain dense hidden states per node over time, leading to high memory and computation costs. \emph{Attention-based models} (e.g., TGAT \cite{tgat}, TGN \cite{benTemporal}) use temporal self-attention to weight historical interactions, but their complexity often scales quadratically with neighborhood size and temporal horizon. More recently, \emph{state-space models} \cite{yuanDGMambaRobustEfficient2025} and \emph{frequency-domain methods} \cite{tianFreeDyGFrequencyEnhanced2024} achieve linear-time complexity; however, they typically rely on global temporal encodings that can miss localized, fine-grained structural dynamics. Together, these limitations reflect a persistent tension between expressiveness and scalability in DGRL.

Spiking Neural Networks (SNNs) offer an alternative computational paradigm suited for temporal modeling. Through sparse, binary spike events, SNNs enable event-driven computation with reduced memory and energy cost \cite{roySpikebasedMachineIntelligence2019}. Motivated by these properties, recent work has explored \emph{Spiking Graph Neural Networks} (SGNNs) for dynamic graph learning. Early studies such as SpikeNet \cite{spikenet23} and Dy-SIGN \cite{dysign24} showed that spike-based propagation improves scalability on large graphs. However, recent benchmarks indicate that SGNNs often underperform continuous-valued GNNs, pointing to unresolved representational bottlenecks \cite{zhangSGNNBenchHolisticEvaluation2025}. A closer analysis of prior SGNNs reveals a trilemma. First, SpikeNet-style sequential spike propagation is efficient but lacks adaptive temporal weighting and long-range dependency modeling. Second, Dy-SIGN improves memory efficiency via equilibrium-based training, but assumes convergence to a fixed point and may lose fine-grained temporal information due to binary encoding. Third, Delay-DSGN~\cite{Delay_DSGN_25} introduces learnable delay kernels but relies on local aggregation, limiting its ability to capture global context.

No existing approach simultaneously achieves (i) biologically plausible spike-based computation, (ii) adaptive multi-scale temporal aggregation integrating spatial and temporal information, and (iii) learnable neuron dynamics that stabilize training without equilibrium assumptions. To address this gap, we propose \textbf{ChronoSpike}, a spiking framework for DGRL that combines adaptive spiking neurons with expressive spatial and temporal aggregation. ChronoSpike has three design components. \textbf{(i)} It uses adaptive Leaky Integrate-and-Fire (LIF) neurons with learnable membrane time constants and firing thresholds, enabling stable temporal dynamics. \textbf{(ii)} It applies multi-head attentive spatial aggregation on continuous node features before spike generation, allowing adaptive neighbor weighting under evolving graphs. \textbf{(iii)} It integrates spike representations over time using a lightweight Transformer-based temporal model with learned positional encodings, capturing long-range dependencies beyond sequential spike accumulation. These components enable ChronoSpike to model local structure and global temporal context while preserving the efficiency of spike-based computation.

Our primary contributions are:
\textbf{(i)} We propose \textbf{ChronoSpike}, a spiking framework for dynamic graph representation learning that integrates adaptive spiking with spatial and temporal aggregation.
\textbf{(ii)} We introduce adaptive LIF neurons with learnable membrane constants and firing thresholds, enabling stable and flexible temporal modeling without equilibrium assumptions.
\textbf{(iii)} We design a hybrid architecture combining multi-head attentive spatial aggregation with a lightweight Transformer-based temporal encoder to capture local structural patterns and long-range temporal dependencies.
\textbf{(iv)} Experiments on large-scale dynamic graph benchmarks show that ChronoSpike outperforms evaluated spiking and non-spiking methods while maintaining favorable computational efficiency.

\section{Related Work}
\textbf{Dynamic Graph Representation Learning. }
DGRL has evolved through several paradigms with distinct scalability-fidelity trade-offs. Recent continuous-time models achieve strong performance but face trade-offs between complexity and accuracy. DyGFormer~\cite{yuBetterDynamicGraph2023} employs a Transformer, whereas GraphMixer~\cite{cong2023graphmixer} uses an MLP-Mixer; both rely on fixed-length history windows, limiting fine-grained temporal resolution and long horizons. FreeDyG~\cite{tianFreeDyGFrequencyEnhanced2024} exploits frequency-domain representations for linear scaling, whereas DG-Mamba~\cite{yuanDGMambaRobustEfficient2025} adapts selective state-space models for efficiency, but both rely on global temporal encodings that fail to capture localized structural shifts in event-driven graphs~\cite{fengComprehensiveSurveyDynamic2026,yiTGBSeqBenchmarkChallenging2025}. \textit{Memory-augmented architectures} maintain explicit node states but incur storage and sequential dependencies. EvolveGCN~\cite{pareja2020evolvegcn} evolves GCN parameters via RNNs, while DyRep~\cite{trivedi2019dyrep} and JODIE~\cite{kumar2019predicting_jodie} update latent representations through temporal point processes. These methods require $O(|V|d)$ memory and suffer gradient pathologies over long sequences~\cite{hochreiter1997long}. \textit{Attention-based variants}, such as TGAT~\cite{tgat}, TGN~\cite{benTemporal}, and TREND~\cite{wen2022trend}, mitigate vanishing gradients via self-attention over temporal neighborhoods but incur costs scaling with graph density and neighborhood size. Neural Temporal Walks~\cite{jinNeuralTemporalWalks2022} incorporate higher-order motifs via learnable walk policies, improving expressiveness at the cost of sampling overhead. Hyperbolic embeddings~\cite{xuScalableEffectiveTemporal2025} capture hierarchical structure but assume quasi-static metrics, violated in networks with rapid topological change~\cite{jiaoSurveyTemporalInteraction2025}. \textit{Classical methods} based on tensor decompositions~\cite{acar2008unsupervised}, incremental spectral updates~\cite{chen2015fast,zhang2018timers}, and temporal random walks~\cite{grover2016node2vec,perozzi2014deepwalk,nguyen2018continuous,mitrovic2019dyn2vec} offer efficiency but lack capacity for complex tasks and impose restrictive assumptions (e.g., low-rank stability, Markovian transitions). 
Temporal graph embedding methods such as LTGE and LTGEInc~\cite{LTGE} improve scalability through matrix-based representations. LTGE constructs and factorizes sparse temporal matrices, while LTGEInc updates embeddings using incremental SVD. These limitations motivate integrating spiking dynamics with learnable temporal aggregation to achieve efficiency without sacrificing expressiveness.

\textbf{Spiking Neural Networks. }
SNNs provide event-driven computation with binary activations, enabling energy efficiency on neuromorphic hardware~\cite{roySpikebasedMachineIntelligence2019}. Surrogate gradient methods enable backpropagation by approximating the non-differentiable Heaviside function~\cite{wu2021training,meng2022training-ICCV,xiao2021training}. This creates a forward-backward mismatch that destabilizes optimization, especially in deep networks~\cite{thiele2020spikegrad,bellec2018long}. Adaptive spiking neurons with learnable membrane time constants~\cite{fangIncorporatingLearnableMembrane2021}, adaptive reset mechanisms~\cite{huangARLIFAdaptiveReset2025,zhangDALIFDualAdaptive2025,yanTrainingHighPerformance2025}, or multi-threshold dynamics~\cite{yangMMTSNNMarkovianDecision2026} improve temporal expressiveness but introduce hyperparameters sensitive to data and initialization. Spiking Transformers extend attention to spike trains. QKFormer~\cite{zhouQKFormerHierarchicalSpiking2024} processes Q-K attention hierarchically across spiking layers. Recent studies~\cite{guoSpikingTransformerIntroducing2025,xiaoRethinkingSpikingSelfAttention2025} introduce addition-only attention for hardware efficiency. Spatial-temporal variants~\cite{leeSpikingTransformerSpatialTemporal2025,zhangSTAASNNSpatialTemporalAttention2025,yuFSTASNNFrequencyBasedSpatialTemporal2025,LU2025ests} partition heads to reduce redundancy. Conversion-based approaches~\cite{diehl2015fast,rueckauer2017conversion,deng2021optimal} map trained ANNs to SNNs but require long inference windows (often hundreds of steps), reducing efficiency gains. These methods show SNN potential for static tasks but expose challenges in temporal credit assignment and latency-accuracy trade-offs, amplified in dynamic graphs.

\textbf{Spiking Neural Networks for Dynamic Graphs. }
Applications of SNNs to graph data mostly target static graphs or limited temporal dynamics. Static spiking GCNs~\cite{zhu2022spiking,xu2021exploiting} perform sparse message passing but lack temporal evolution mechanisms. Knowledge graph embeddings with spike-based encodings~\cite{chian2021learning,DBLP:conf/ijcnn/DoldG21} and synaptic delay models~\cite{xiao2024temporal} capture relational structure but do not model continuous temporal dependencies. Riemannian spiking GNNs~\cite{sunSpikingGraphNeural2024} extend to non-Euclidean geometry but focus on static manifolds. For \textit{dynamic graphs}, SpikeNet~\cite{spikenet23} replaces RNN encoders with spiking neurons, processes snapshots, and updates hidden states via spike accumulation. While this reduces memory compared to continuous RNNs, it propagates information only from the previous step, failing to aggregate long-range dependencies or adaptively weight historical snapshots. Dy-SIGN~\cite{dysign24} reduces memory via implicit differentiation at equilibrium~\cite{bai2019deep,gu2020implicit}, instead of full backpropagation through time, and uses early-layer skip connections to mitigate information loss from binary propagation. However, Dy-SIGN assumes convergence to a fixed point, which may fail in non-stationary graphs. Delay-DSGN~\cite{Delay_DSGN_25} introduces learnable Gaussian delay kernels for aggregation to stabilize gradients, but relies on fixed-form priors that restrict temporal modeling to local shifts and limit global dependency capture.
Methods such as SiGNN~\cite{chen2025signn} and HASNN~\cite{hasnn} demonstrate the potential of spiking models in temporal graph settings; however, ChronoSpike differs by integrating adaptive LIF neurons, continuous-feature spatial attention, and Transformer-based temporal aggregation within a unified framework. 
Recent studies on adversarial robustness~\cite{daiMemFreezingNovelAdversarial,lunEffectiveSparseAdversarial2025} and adaptive dynamics~\cite{baronigAdvancingSpatiotemporalProcessing2025,zhaoDynamicReactiveSpiking2024} expose SNN vulnerabilities and their capacity for temporal adaptation. Signal-SGN++~\cite{zhengSignalSGNTopologyEnhancedTimeFrequency2025} combines time-frequency decomposition with spiking graph convolution for skeleton-based action recognition, showing the value of multi-scale temporal representations in spiking models. SGNNBench~\cite{zhangSGNNBenchHolisticEvaluation2025} evaluates 18 datasets and shows that spiking models often underperform continuous GNNs on large-scale benchmarks, especially in balancing accuracy and efficiency.

\begin{figure}[!t]
  \centering
  \includegraphics[width=0.8\linewidth, trim=0.5cm 0 1.8cm 0,
  clip]{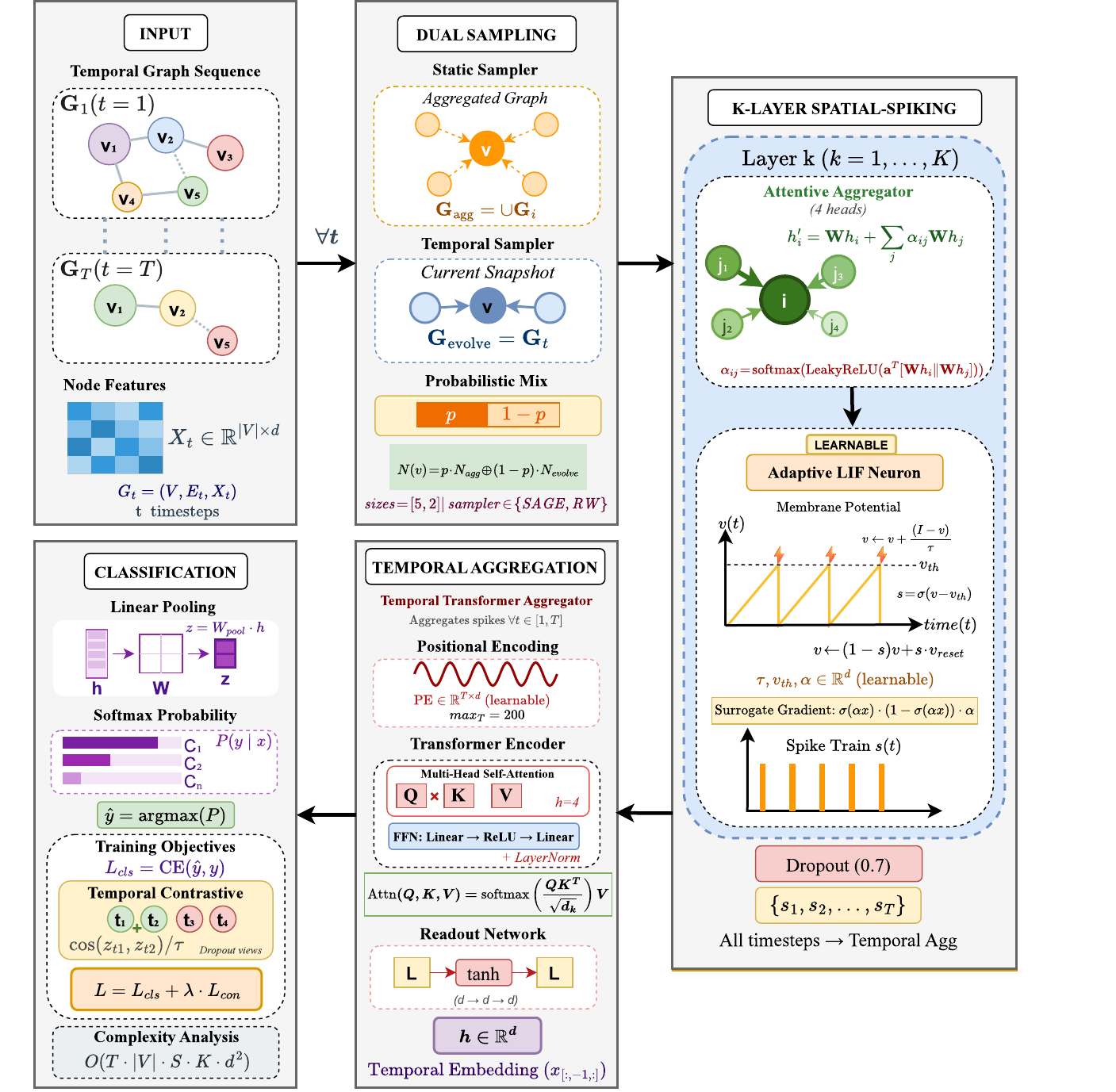}
  \caption{Overview of ChronoSpike. The dynamic graph is represented as a sequence of snapshots. At each time step, node features are aggregated from sampled neighborhoods using a multi-head attentive spatial aggregator and encoded into spike signals via adaptive LIF neurons. Temporal dependencies across snapshots are captured by a lightweight Transformer-based temporal aggregation module with learnable positional encodings. The final node representations are used for prediction.}
  \label{fig:chronospike}
\end{figure}

\section{Problem Definition}
We study DGRL, where the input is a sequence of graph snapshots at discrete time steps, with evolving topology and node attributes. Formally, a dynamic graph is a sequence of $T$ snapshots capturing network states at times $1,\dots,T$. The objective is to learn a time-dependent embedding for each node that encodes its structural context at each time and the temporal evolution of its interactions. We focus on temporal node classification, where dynamic embeddings predict node labels at future time steps. Effective representations must be discriminative, capturing both who and when a node interacts, to enable accurate, time-dependent classification. The full formulation is provided in Appendix~\ref{app:problem}.

\section{ChronoSpike}
\label{sec:chronospike}
We propose \textbf{ChronoSpike}, an adaptive spiking framework for dynamic graph representation learning that addresses three fundamental limitations of existing approaches: sequential spike-propagation methods cannot capture long-range dependencies, binary encoding schemes lose information through equilibrium assumptions, and local aggregation kernels lack global temporal context. ChronoSpike integrates three design components to achieve temporal expressiveness, computational efficiency, and training stability: \emph{(i)} adaptive LIF neurons with learnable membrane time constants and firing thresholds, enabling flexible temporal dynamics without equilibrium constraints; \emph{(ii)} multi-head attentive spatial aggregation on continuous features before spike encoding, preserving fine-grained information while introducing sparsity; and \emph{(iii)} a lightweight Transformer-based temporal encoder with learned positional encodings, capturing long-range dependencies with $O(T \cdot d)$ activation/state memory plus an $O(T^2)$ per-node self-attention term that is bounded for the horizons we evaluate (see Appendix~\ref{app:complexity}). Figure~\ref{fig:chronospike} illustrates the complete architecture. We provide formal complexity analysis, stability guarantees, and convergence proofs in Appendix~\ref{app:theoretical_analysis}, with algorithmic specifications in Appendix~\ref{app:algorithm}.

\subsection{Problem Setting and Notation}
We model a dynamic graph as a sequence of $T$ snapshots $G={G_1,\dots,G_T}$, where each snapshot $G_t=(V,E_t,X_t)$ has a fixed node set $V$ ($|V|=N$), a time-dependent edge set $E_t$, and node features $X_t\in\mathbb{R}^{N\times d}$. For a node $v\in V$, its feature at time $t$ is $\mathbf{x}_v^{(t)}$, and its neighborhood is $\mathcal{N}_t(v)=\{u\in V:(u,v)\in E_t\}$. While the node set is fixed, both node features and connectivity evolve over time. The objective is to learn a mapping $\Phi_{\Theta}:(X_1,E_1,\dots,X_T,E_T)\mapsto Z$, where $Z\in\mathbb{R}^{N\times d'}$ contains a $d'$-dimensional embedding per node, with $d'\ll d$. The embeddings should encode structural and temporal information and support downstream tasks, including temporal node classification. Full formulation, assumptions, and notation are in Appendix~\ref{app:notation}.

\subsection{Dynamic Graph Sampling and Spatial Encoding}
\label{subsec:sampling}
ChronoSpike uses inductive neighborhood sampling at each time step to ensure scalability on large dynamic graphs. For a target node $v$, a fixed-size neighbor set is formed by sampling from the cumulative adjacency and newly formed edges at timestep $t$, where a ratio $p$ controls the proportion from each source. This hybrid strategy captures long-term structural context and recent temporal changes while bounding snapshot cost and supporting inductive generalization.

Given the sampled neighborhood, ChronoSpike performs spatial aggregation with a learnable aggregation operator. For node $v$ at time $t$, the intermediate spatial representation is
\begin{equation}
  \mathbf{h}_v^{(t)} = \mathbf{W}_s \mathbf{x}_v^{(t)} \;+\;
  \sum_{u \in \tilde{\mathcal{N}}_t(v)} \alpha_{vu}^{(t)} \mathbf{W}_n \mathbf{x}_u^{(t)}
  \label{eq:spatial}
\end{equation}
where $\mathbf{x}_v^{(t)}$ is the input feature, $\mathbf{W}_s, \mathbf{W}_n \in \mathbb{R}^{d_h \times d}$ are learnable weights, and $\tilde{\mathcal{N}}_t(v)$ denotes the fixed-size neighborhood obtained by the inductive sampling procedure described above (Appendix~\ref{app:spatial_encoding}, Eq.~\ref{eq:8}). The attention coefficients $\alpha_{vu}^{(t)}$ are computed per head using learned projections and normalized over the sampled neighborhood. This mechanism adaptively weights neighbors under evolving topologies while avoiding full adjacency matrix multiplications, maintaining efficiency on large graphs. See Appendix~\ref{app:spatial_encoding} for full sampling and attention details.

\subsection{Spike Encoding via Adaptive LIF Neurons}
\label{subsec:lif}
Adaptive LIF neurons convert aggregated spatial representations into sparse temporal spikes. For each node $v$ and hidden dimension $i$, the membrane potential evolves as
\begin{equation}
  m_{v,i}^{(t)} = m_{v,i}^{(t-1)} + \frac{1}{\tau_i}
  \left( h_{v,i}^{(t)} - (m_{v,i}^{(t-1)} - m_{\mathrm{reset}}) \right)
  \label{eq:lif}
\end{equation}
where $\tau_i$ is a learnable time constant, $m_{\mathrm{reset}}$ is the reset potential, and $h_{v,i}^{(t)}$ is the synaptic input from Eq.~\eqref{eq:spatial}. Neuron parameters are learned per feature channel, enabling heterogeneous temporal dynamics. A spike is emitted when the membrane potential exceeds a learnable threshold $V_{\mathrm{th},i}$:
\begin{equation}
  s_{v,i}^{(t)} = \mathbb{I}\big(m_{v,i}^{(t)} \ge V_{\mathrm{th},i}\big)
\end{equation}
After firing, the membrane potential is reset to prevent unbounded accumulation. This adaptive LIF design enables heterogeneous temporal dynamics while producing sparse spike representations. During training, a sigmoid-based surrogate gradient approximates the non-differentiable spike function, enabling stable end-to-end optimization. Appendix~\ref{app:adaptive_lif} provides the full LIF update, reset, and training details.

\subsection{Temporal Spike Integration}
\label{subsec:temporal}
ChronoSpike integrates temporal information by aggregating spike representations across graph snapshots. The spike outputs ${\mathbf{s}_v^{(1)}, \dots, \mathbf{s}_v^{(T)}}$ are stacked into a sequence and augmented with learnable positional encodings to preserve order. The sequence is processed by a lightweight Transformer encoder to capture long-range dependencies:
\begin{equation}
  \mathbf{Z}_{v} = \mathrm{Transformer}\big(\mathbf{s}_{v}^{(1:T)} + \mathbf{P}_{1:T}\big)
\end{equation}
where $\mathbf{P}_{1:T}$ are positional encodings. The embedding $\mathbf{z}_v$ from the last timestep is used for downstream prediction. This design decouples temporal modeling from recurrent state propagation and avoids storing dense hidden states, enabling learning over long horizons with controlled memory use. See Appendix~\ref{app:temporal_integration} for spike sequence encoding and Transformer formulation.

\subsection{Joint Optimization with Contrastive Regularization}
\label{subsec:optimization}
Given the final node embedding $\mathbf{z}_v$, ChronoSpike produces predictions using a linear classifier followed by softmax:
\begin{equation}
  \hat{\mathbf{y}}_v = \mathrm{Softmax}(\mathbf{W}_c \mathbf{z}_v + \mathbf{b}_c)
\end{equation}
where $\mathbf{W}_c \in \mathbb{R}^{C \times d'}$ is a learnable weight matrix, $\mathbf{b}_c \in \mathbb{R}^C$ is the classifier bias, and $C$ is the number of classes. The model is trained end-to-end with supervised cross-entropy over labeled nodes. ChronoSpike applies temporal contrastive regularization for representation robustness. For each embedding $\mathbf{z}_v$, two stochastic views are generated via feature dropout, and a contrastive loss enforces consistency between them. Let $\mathbf{z}_{v}$ and $\mathbf{z}_{v}'$ denote the two views. The contrastive objective is:
\begin{equation}
  \mathcal{L}_{\mathrm{con}} = - \sum{v \in \mathcal{V}}
  \log \frac{\exp\big(\mathrm{sim}(\mathbf{z}_{v}, \mathbf{z}_{v}') / \tau \big)}
  {\sum_{u \in \mathcal{V}} \exp\big(\mathrm{sim}(\mathbf{z}_{v}, \mathbf{z}_{u}') / \tau \big)}
\end{equation}
where $\mathrm{sim}(\cdot,\cdot)$ is cosine similarity and $\tau$ is a temperature; in practice the sums are evaluated over the current mini-batch $\mathcal{V}_b\subseteq\mathcal{V}$ rather than the full node set, matching Algorithm~\ref{alg:chronospike-training}. The overall training objective is:
\begin{equation}
  \mathcal{L} = \mathcal{L}_{\mathrm{cls}} + \lambda \mathcal{L}_{\mathrm{con}}
\end{equation}
where $\mathcal{L}_{\mathrm{cls}}$ is cross-entropy and $\lambda$ weights the contrastive term, which is used only during training and not inference. Since spike generation involves non-differentiable thresholds, ChronoSpike uses surrogate gradients in backpropagation. Spike derivative is approximated by a smooth surrogate, while the forward pass preserves spike dynamics, enabling stable optimization without altering signal propagation. Full loss definitions and training setup are in Appendix~\ref{app:optimization}.

\section{Experiments}
\label{sec:exp}
\subsection{Experimental Setup}

\begin{table}[h]
  \centering
  \caption{Statistics of datasets used in our experiments.}
  \label{tab:dataset-stats}
  \resizebox{0.5\linewidth}{!}{
    \begin{tabular}{lrrrr}
      \toprule
      \textbf{Dataset} & \textbf{\#Nodes} & \textbf{\#Edges} & \textbf{\#Steps} & \textbf{\#Classes} \\
      \midrule
      DBLP   & 28,085     & 236,894     & 27  & 10 \\
      Tmall  & 577,314    & 4,807,545   & 186 & 5  \\
      Patent & 2,738,012  & 13,960,811  & 25  & 6  \\
      \bottomrule
    \end{tabular}
  }
\end{table}

\subsubsection{Datasets}
\label{sec:datasets}

We evaluate ChronoSpike on three real-world dynamic graph datasets: DBLP, Tmall~\cite{MMDNE_dblp_and_tmall}, and Patent~\cite{patent}. Dataset details are in Appendix~\ref{app:datasets}. These datasets vary in scale, temporal granularity, and structural dynamics, from moderately sized graphs with short temporal depth to large graphs with millions of nodes and long horizons. Summary statistics are in Table~\ref{tab:dataset-stats}.

\subsubsection{Evaluation Metrics}\label{sec:metrics}
We evaluate temporal node classification using Macro-F1 and Micro-F1, standard metrics in DGRL. Let $\mathcal{C}$ be the set of classes. For each class $c\in\mathcal{C}$,
$\mathrm{Precision}_c=\frac{\mathrm{TP}_c}{\mathrm{TP}_c+\mathrm{FP}_c}$ and
$\mathrm{Recall}_c=\frac{\mathrm{TP}_c}{\mathrm{TP}_c+\mathrm{FN}_c}$,
where $\mathrm{TP}_c$, $\mathrm{FP}_c$, and $\mathrm{FN}_c$ are true positives, false positives, and false negatives. The class-wise F1 score, $\mathrm{F1}_c=\frac{2 \cdot \mathrm{Precision}_c \cdot \mathrm{Recall}_c}{\mathrm{Precision}_c+\mathrm{Recall}_c}$.
$\mathrm{Macro\text{-}F1}=\frac{1}{|\mathcal{C}|}\sum_{c\in\mathcal{C}}\mathrm{F1}_c$, while $\mathrm{Micro\text{-}F1}=\frac{2\sum_c \mathrm{TP}_c}{2\sum_c \mathrm{TP}_c+\sum_c \mathrm{FP}_c+\sum_c \mathrm{FN}_c}$.
Efficiency metrics, including trainable parameters and average training time per epoch under identical settings, are also reported. (Details in Appendix~\ref{app:metrics}.)

\subsubsection{Implementation Details}
We compare ChronoSpike with twelve state-of-the-art baselines (Details in Appendix \ref{app:baseline}), including static graph methods DeepWalk \cite{perozzi2014deepwalk} and Node2Vec \cite{grover2016node2vec}; shallow dynamic graph methods HTNE \cite{HTNE}, M$^2$DNE \cite{MMDNE_dblp_and_tmall}, and DyTriad \cite{DynamicTriad}; neural dynamic graph methods, including message passing and recurrent architectures MPNN \cite{MPNN}, JODIE \cite{kumar2019predicting_jodie}, and EvolveGCN \cite{pareja2020evolvegcn}; attention-based methods TGAT \cite{tgat}; and SGNNs SpikeNet \cite{spikenet23}, Dy-SIGN \cite{dysign24}, and Delay-DSGN \cite{Delay_DSGN_25}. ChronoSpike is trained with mini-batch stochastic optimization using AdamW for 100 epochs. The batch size is fixed to 1024 for all datasets. Learning rates are dataset-specific and fixed across runs: $5\times10^{-3}$ for DBLP and Tmall, and $1\times10^{-2}$ for Patent. During inference, batch sizes depend on dataset scale: 10,000 for Patent and 200,000 for DBLP and Tmall. Complete implementation details are provided in Appendix \ref{app:implementation}.

\begin{table}[!t]
  \centering
  \caption{Macro and Micro F1 scores (\%) of node classification tasks with different training ratios. The results are averaged over five runs. The metric is reported as the mean ± standard deviation, calculated over five experimental runs. \textbf{Bold} and \underline{underlined} values indicate the best and second-best performance in each column, respectively. The symbol ``-" denotes time-consuming.}
  \label{perf_comparison}
  \setlength{\tabcolsep}{1.3mm}
  \resizebox{\linewidth}{!}{
    \begin{tabular}{ccc ccccccccccccc}
      \toprule[1.5pt]
      {\textbf{Datasets}}                                 & {\textbf{Metrics}}       & \textbf{Tr. Ratio} & \textbf{DeepWalk} & \textbf{Node2Vec} & \textbf{HTNE} & \textbf{M$^2$DNE} & \textbf{DyTriad} & \textbf{MPNN}   & \textbf{JODIE}  & \textbf{EvolveGCN} & \textbf{TGAT}   & \textbf{SpikeNet} & \textbf{Dy-SIGN}            & \textbf{Delay-DSGN}                            & \cellcolor{gray!32}\textbf{ChronoSpike}        \\
      \midrule[1pt]
      \multirow{6}{*}{\textcolor{mygreen}{\textbf{DBLP}}} &
      \multirow{3}{*}{{Ma-F1}}                            & 40\%                     & 67.08              & 66.07             & 67.68             & 69.02         & 60.45             & 64.19${\pm0.4}$  & 66.73${\pm1.0}$ & 67.22${\pm0.3}$ & 71.18${\pm0.4}$    & 70.88${\pm0.4}$ & 70.94${\pm0.1}$   & \underline{72.32${\pm0.4}$} & \cellcolor{mygreen!25}\textbf{75.84${\pm0.3}$}                                                  \\
      &                          & 60\%               & 67.17             & 66.81             & 68.24         & 69.48             & 64.77            & 63.91${\pm0.3}$ & 67.32${\pm1.1}$ & 69.78${\pm0.8}$    & 71.74${\pm0.5}$ & 71.98${\pm0.3}$   & 72.07${\pm0.1}$             & \underline{74.16${\pm0.3}$}                    & \cellcolor{mygreen!25}\textbf{77.79${\pm0.4}$} \\
      &                          & 80\%               & 67.12             & 66.93             & 68.36         & 69.75             & 66.42            & 65.05${\pm0.5}$ & 67.53${\pm1.3}$ & 71.20${\pm0.7}$    & 72.15${\pm0.3}$ & 74.65${\pm0.5}$   & 74.67${\pm0.5}$             & \underline{76.54${\pm0.4}$}                    & \cellcolor{mygreen!25}\textbf{79.13${\pm0.3}$} \\
      \cmidrule{2-16}
      & \multirow{3}{*}{{Mi-F1}} & 40\%               & 66.53             & 66.80             & 67.68         & 69.23             & 65.13            & 65.72${\pm0.4}$ & 68.44${\pm0.6}$ & 69.12${\pm0.8}$    & 71.10${\pm0.2}$ & 71.98${\pm0.3}$   & 71.90${\pm0.1}$             & \underline{72.56${\pm0.2}$}                    & \cellcolor{mygreen!25}\textbf{76.80${\pm0.2}$} \\
      &                          & 60\%               & 66.89             & 67.37             & 68.24         & 69.47             & 66.80            & 66.79${\pm0.6}$ & 68.51${\pm0.8}$ & 70.43${\pm0.6}$    & 71.85${\pm0.4}$ & 72.35${\pm0.8}$   & 72.61${\pm0.4}$             & \underline{74.44${\pm0.3}$}                    & \cellcolor{mygreen!25}\textbf{78.47${\pm0.3}$} \\
      &                          & 80\%               & 66.38             & 67.31             & 68.36         & 69.71             & 66.95            & 67.74${\pm0.3}$ & 68.80${\pm0.9}$ & 71.32${\pm0.5}$    & 73.12${\pm0.3}$ & 74.86${\pm0.5}$   & 74.96${\pm0.2}$             & \underline{76.87${\pm0.5}$}                    & \cellcolor{mygreen!25}\textbf{79.33${\pm0.1}$} \\
      \midrule
      \multirow{6}{*}{\textcolor{red}{\textbf{Tmall}}}    & \multirow{3}{*}{{Ma-F1}} & 40\%               & 49.09             & 54.37             & 54.81         & 57.75             & 44.98            & 47.71${\pm0.8}$ & 52.62${\pm0.8}$ & 53.02${\pm0.7}$    & 56.90${\pm0.6}$ & 58.84${\pm0.4}$   & 57.48${\pm0.1}$             & \underline{60.25${\pm0.1}$}                    & \cellcolor{red!25}\textbf{60.80${\pm0.4}$}     \\
      &                          & 60\%               & 49.29             & 54.55             & 54.89         & 57.99             & 48.97            & 47.78${\pm0.7}$ & 54.02${\pm0.6}$ & 54.99${\pm0.7}$    & 57.61${\pm0.7}$ & 61.13${\pm0.8}$   & 60.94${\pm0.2}$             & \underline{62.56${\pm0.3}$}                    & \cellcolor{red!25}\textbf{62.86${\pm0.3}$}     \\
      &                          & 80\%               & 49.53             & 54.58             & 54.93         & 58.47             & 51.16            & 50.27${\pm0.5}$ & 54.17${\pm0.2}$ & 55.78${\pm0.6}$    & 58.01${\pm0.7}$ & 62.40${\pm0.6}$   & 61.89${\pm0.1}$             & \underline{64.02${\pm0.2}$}                    & \cellcolor{red!25}\textbf{64.74${\pm0.2}$}     \\
      \cmidrule{2-16}
      & \multirow{3}{*}{{Mi-F1}} & 40\%               & 57.11             & 60.41             & 54.81         & 64.21             & 53.24            & 57.82${\pm0.7}$ & 58.36${\pm0.5}$ & 59.96${\pm0.7}$    & 62.05${\pm0.5}$ & 63.52${\pm0.7}$   & 62.93${\pm0.3}$             & \underline{64.32${\pm0.1}$}                    & \cellcolor{red!25}\textbf{66.19${\pm0.6}$}     \\
      &                          & 60\%               & 57.34             & 60.56             & 54.89         & 64.38             & 56.88            & 57.66${\pm0.5}$ & 60.28${\pm0.3}$ & 61.19${\pm0.6}$    & 62.92${\pm0.4}$ & 64.84${\pm0.4}$   & 64.10${\pm0.3}$             & \underline{66.20${\pm0.2}$}                    & \cellcolor{red!25}\textbf{67.79${\pm0.3}$}     \\
      &                          & 80\%               & 57.88             & 60.66             & 54.93         & 64.65             & 60.72            & 58.07${\pm0.6}$ & 60.49${\pm0.3}$ & 61.77${\pm0.6}$    & 63.32${\pm0.7}$ & 66.10${\pm0.3}$   & 65.82${\pm0.2}$             & \underline{67.88${\pm0.4}$}                    & \cellcolor{red!25}\textbf{69.05${\pm0.4}$}     \\
      \midrule
      \multirow{6}{*}{\textcolor{blue}{\textbf{Patent}}}  & \multirow{3}{*}{{Ma-F1}} & 40\%               & 72.32${\pm0.9}$   & 69.01${\pm0.9}$   & -             & -                 & -                & -               & 77.57${\pm0.8}$ & 79.67${\pm0.4}$    & 81.51${\pm0.4}$ & 83.53${\pm0.6}$   & 83.57${\pm0.3}$             & \underline{83.72${\pm0.1}$}                    & \cellcolor{blue!25}\textbf{86.20${\pm0.3}$}    \\
      &                          & 60\%               & 72.25${\pm1.2}$   & 69.08${\pm0.9}$   & -             & -                 & -                & -               & 77.69${\pm0.6}$ & 79.76${\pm0.5}$    & 81.56${\pm0.6}$ & 83.85${\pm0.7}$   & 83.77${\pm0.2}$             & \underline{84.01${\pm0.1}$}                    & \cellcolor{blue!25}\textbf{86.34${\pm0.2}$}    \\
      &                          & 80\%               & 72.05${\pm1.1}$   & 68.99${\pm1.0}$   & -             & -                 & -                & -               & 77.67${\pm0.4}$ & 80.13${\pm0.4}$    & 81.57${\pm0.5}$ & 83.90${\pm0.6}$   & 83.91${\pm0.2}$             & \underline{84.20${\pm0.1}$}                    & \cellcolor{blue!25}\textbf{86.09${\pm0.1}$}    \\
      \cmidrule{2-16}
      & \multirow{3}{*}{{Mi-F1}} & 40\%               & 71.57${\pm1.3}$   & 68.14${\pm0.9}$   & -             & -                 & -                & -               & 77.64${\pm0.7}$ & 79.39${\pm0.5}$    & 80.79${\pm0.7}$ & 83.48${\pm0.8}$   & 83.50${\pm0.2}$             & \underline{83.66${\pm0.1}$}                    & \cellcolor{blue!25}\textbf{86.03${\pm0.3}$}    \\
      &                          & 60\%               & 71.53${\pm1.0}$   & 68.20${\pm0.7}$   & -             & -                 & -                & -               & 77.89${\pm0.5}$ & 79.75${\pm0.3}$    & 80.81${\pm0.6}$ & 83.80${\pm0.7}$   & 83.47${\pm0.1}$             & \underline{83.97${\pm0.1}$}                    & \cellcolor{blue!25}\textbf{86.16${\pm0.2}$}    \\
      &                          & 80\%               & 71.38${\pm1.2}$   & 68.10${\pm0.5}$   & -             & -                 & -                & -               & 77.93${\pm0.4}$ & 80.01${\pm0.3}$    & 80.93${\pm0.6}$ & 83.88${\pm0.9}$   & 83.90${\pm0.2}$             & \underline{84.15${\pm0.1}$}                    & \cellcolor{blue!25}\textbf{85.94${\pm0.1}$}    \\
      \bottomrule[1.5pt]
    \end{tabular}
  }
\end{table}

\begin{table}[ht]
  \centering

  \caption{Ablation study of ChronoSpike for node classification on DBLP, Tmall, and Patent datasets. Tr. Ratio denotes the training ratio. For ChronoSpike, we report mean ± standard deviation over five runs. Ablation results are reported as single runs, following common practice.}
  \label{tab:ablation_node_cl}
  \setlength{\tabcolsep}{1.3mm}
  \resizebox{\linewidth}{!}{
    \begin{tabular}{ccc cccccccc}
      \toprule[1.5pt]
      \textbf{Datasets} & \textbf{Metrics}                        & \textbf{Tr. Ratio}
      & w/o Adaptive LIF
      & w/o Temporal Transformer
      & w/o Attention
      & w/o Contrastive Loss
      & w/o SNN (ReLU)
      & LSTM Temporal
      & Static (No Temporal)
      & \cellcolor{gray!32}\textbf{ChronoSpike}                                                                                                                                                                                                         \\
      \midrule[1pt]
      \multirow{6}{*}{\textcolor{mygreen}{\textbf{DBLP}}}
      & \multirow{3}{*}{Ma-F1}
      & 40\%                                    & 0.7353             & 0.7084      & 0.7412 & 0.7431      & \ul{0.7480} & 0.7305      & 0.5104      & \cellcolor{mygreen!25}\textbf{0.7584${\pm0.003}$}                                                   \\
      &                                         & 60\%               & 0.7601      & 0.7237 & 0.7500      & 0.7548      & \ul{0.7666} & 0.7467      & 0.7408                                          & \cellcolor{mygreen!25}\textbf{0.7779${\pm0.004}$} \\
      &                                         & 80\%               & \ul{0.7866} & 0.7548 & 0.7707      & 0.7845      & 0.7860      & 0.7746      & 0.0456                                          & \cellcolor{mygreen!25}\textbf{0.7913${\pm0.003}$} \\
      \cmidrule{2-11}
      & \multirow{3}{*}{Mi-F1}
      & 40\%                                    & 0.7468             & 0.7166      & 0.7476 & 0.7485      & \ul{0.7592} & 0.7361      & 0.6138      & \cellcolor{mygreen!25}\textbf{0.7680${\pm0.002}$}                                                   \\
      &                                         & 60\%               & 0.7622      & 0.7370 & 0.7570      & 0.7585      & \ul{0.7684} & 0.7549      & 0.7454                                          & \cellcolor{mygreen!25}\textbf{0.7847${\pm0.003}$} \\
      &                                         & 80\%               & 0.7859      & 0.7609 & 0.7752      & 0.7842      & \ul{0.7888} & 0.7752      & 0.2938                                          & \cellcolor{mygreen!25}\textbf{0.7933${\pm0.001}$} \\
      \midrule
      \multirow{6}{*}{\textcolor{red}{\textbf{Tmall}}}
      & \multirow{3}{*}{Ma-F1}
      & 40\%                                    & 0.5951             & 0.5666      & 0.5952 & \ul{0.6010} & 0.5983      & 0.5845      & 0.1086      & \cellcolor{red!25}\textbf{0.6080${\pm0.004}$}                                                       \\
      &                                         & 60\%               & 0.6030      & 0.5850 & 0.6006      & 0.6048      & \ul{0.6103} & 0.6048      & 0.1086                                          & \cellcolor{red!25}\textbf{0.6286${\pm0.003}$}     \\
      &                                         & 80\%               & 0.6187      & 0.5970 & 0.6153      & \ul{0.6248} & 0.6082      & 0.6209      & 0.1086                                          & \cellcolor{red!25}\textbf{0.6474${\pm0.002}$}     \\
      \cmidrule{2-11}
      & \multirow{3}{*}{Mi-F1}
      & 40\%                                    & 0.6460             & 0.6232      & 0.6422 & \ul{0.6498} & 0.6486      & 0.6385      & 0.3705      & \cellcolor{red!25}\textbf{0.6619${\pm0.006}$}                                                       \\
      &                                         & 60\%               & 0.6563      & 0.6399 & 0.6513      & \ul{0.6596} & 0.6550      & 0.6499      & 0.3706                                          & \cellcolor{red!25}\textbf{0.6779${\pm0.003}$}     \\
      &                                         & 80\%               & 0.6674      & 0.6518 & 0.6610      & \ul{0.6683} & 0.6592      & 0.6656      & 0.3706                                          & \cellcolor{red!25}\textbf{0.6905${\pm0.004}$}     \\
      \midrule
      \multirow{6}{*}{\textcolor{blue}{\textbf{Patent}}}
      & \multirow{3}{*}{Ma-F1}
      & 40\%                                    & 0.8307             & 0.8338      & 0.8263 & 0.8325      & 0.8269      & \ul{0.8369} & 0.8045      & \cellcolor{blue!25}\textbf{0.8620$\pm0.003$}                                                        \\
      &                                         & 60\%               & 0.8292      & 0.8328 & 0.8253      & 0.8351      & 0.8233      & \ul{0.8379} & 0.8057                                          & \cellcolor{blue!25}\textbf{0.8634$\pm0.002$}      \\
      &                                         & 80\%               & 0.8304      & 0.8330 & 0.8250      & 0.8309      & 0.8312      & \ul{0.8403} & 0.8048                                          & \cellcolor{blue!25}\textbf{0.8609$\pm0.001$}      \\
      \cmidrule{2-11}
      & \multirow{3}{*}{Mi-F1}
      & 40\%                                    & 0.8298             & 0.8321      & 0.8252 & 0.8314      & 0.8277      & \ul{0.8366} & 0.8036      & \cellcolor{blue!25}\textbf{0.8603$\pm0.003$}                                                        \\
      &                                         & 60\%               & 0.8291      & 0.8323 & 0.8238      & 0.8347      & 0.8257      & \ul{0.8379} & 0.8033                                          & \cellcolor{blue!25}\textbf{0.8616$\pm0.002$}      \\
      &                                         & 80\%               & 0.8304      & 0.8326 & 0.8240      & 0.8299      & 0.8305      & \ul{0.8391} & 0.8041                                          & \cellcolor{blue!25}\textbf{0.8594$\pm0.001$}      \\
      \bottomrule[1.5pt]
    \end{tabular}
  }
\end{table}

\subsection{Experimental Results}

\subsubsection{Performance Comparisons}
Table~\ref{perf_comparison} reports Macro-F1 and Micro-F1 for temporal node classification across three datasets with varying training ratios. ChronoSpike consistently outperforms the strongest baseline, Delay-DSGN~\cite{Delay_DSGN_25}, with average gains of \textbf{2.00\% Macro-F1} and \textbf{2.41\% Micro-F1}. Static methods like Deepwalk and Node2Vec underperform since they ignore temporal evolution, while traditional unsupervised dynamic models like HTNE and M$^2$DNE perform well on smaller graphs but struggle on the large Patent dataset due to high computational cost, a common issue for other methods such as MPNN and DyTriad. Supervised dynamic models (JODIE, EvolveGCN, and TGAT) benefit from more data but still lag behind all spike-based models, including ChronoSpike. Prior spiking architectures (SpikeNet, Dy-SIGN, and Delay-DSGN) are efficient, but their performance plateaus due to local kernels, binary information loss, or strictly sequential propagation, which limit the modeling of long-range structural evolution. ChronoSpike achieves consistent Micro-F1 gains of 2--6\% on DBLP, 1--3\% on Tmall, and 2--3\% on Patent over other spiking models, demonstrating scalability across datasets. By combining hybrid sampling, per-channel adaptive spiking, and a global Transformer, ChronoSpike captures fine-grained temporal dependencies and global context while avoiding recurrent-state instability and equilibrium constraints. To complement these results, we visualize training loss curves across all datasets and supervision levels in Figure \ref{fig:loss_curves}. The plots show smooth, consistent convergence behavior, with faster stabilization at higher supervision ratios (Appendix~\ref{app:train_dynamics}).

\subsubsection{Ablation Study}
We conduct ablation studies on DBLP, Tmall, and Patent under different training ratios to assess the contribution of core components of ChronoSpike (Table~\ref{tab:ablation_node_cl}). Removing the temporal Transformer causes the largest drop on DBLP and Tmall (3--5\% Micro- and Macro-F1) and a consistent drop on Patent ($\approx$2.7--3.1\%), confirming its role in capturing long-range temporal dependencies. Replacing adaptive LIF neurons with fixed ones degrades performance across all benchmarks, showing the importance of learnable membrane dynamics. Substituting attentive aggregation with mean pooling results in moderate drops, most notably on dense Tmall and Patent, where it yields among the lowest component-level ablation scores. The contrastive loss provides consistent regularization gains across all settings. The identity of the strongest ablation variant shifts across datasets, the ReLU variant and contrastive loss ablation on DBLP and Tmall, the LSTM temporal variant on Patent, yet ChronoSpike outperforms all ablations uniformly, with margins of 2.0--2.5 percentage points over the best ablation on Patent, demonstrating that sparse spiking dynamics and Transformer-based temporal aggregation provide complementary benefits. The static baseline degrades substantially across the board, confirming that temporal modeling is indispensable regardless of dataset scale. Full results are reported in Appendix~\ref{app:ablation}. We further analyze how ChronoSpike captures temporal dependencies via interpretability studies in Appendix~\ref{app:interpretability}.

\begin{figure}[t]
  \centering
  \subfloat[Parameter size\label{param_size}]{%
  \includegraphics[width=0.3\linewidth]{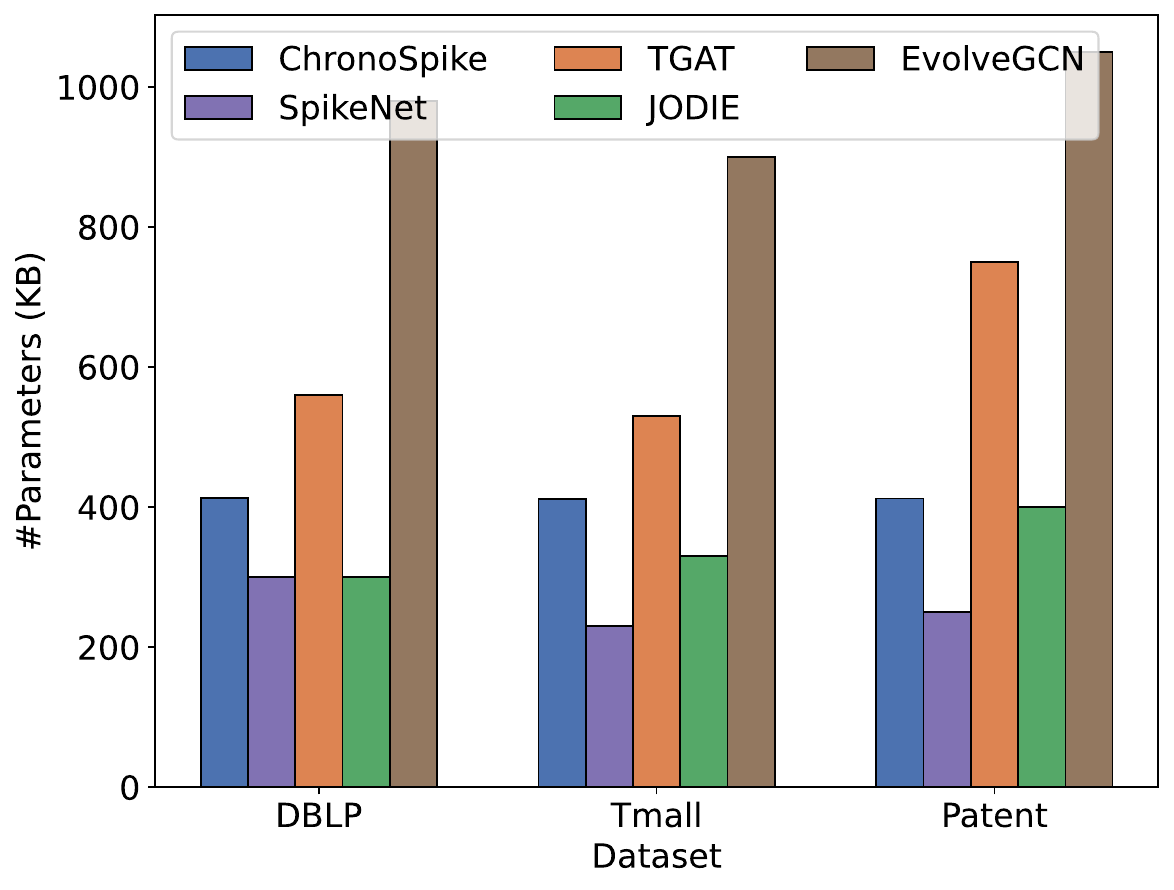}}
  \hspace{15pt}
  \subfloat[Training time\label{train_time}]{%
  \includegraphics[width=0.3\linewidth]{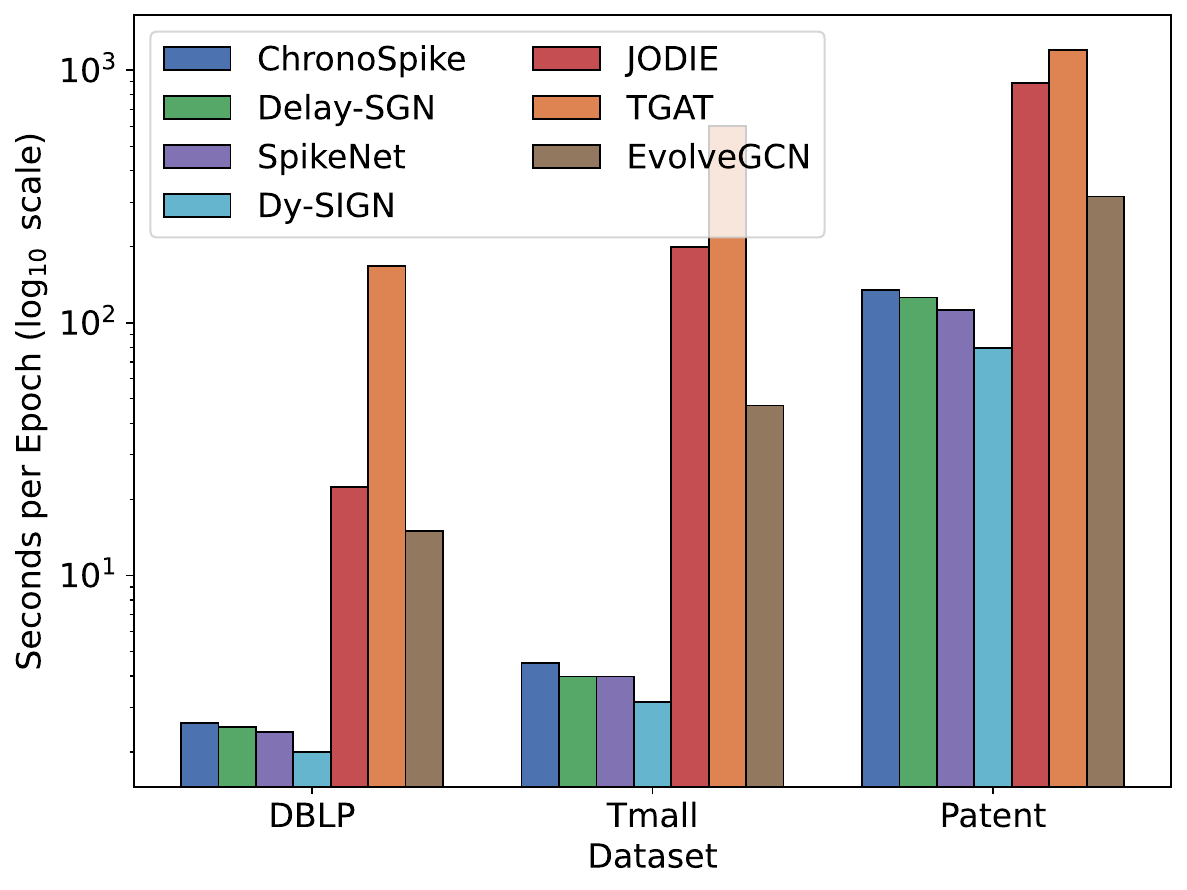}}
  \caption{Overhead comparison of different methods in terms of model parameter size and average training time per epoch. Models that do not scale to the \textit{Patent} dataset or do not report overhead statistics are excluded from the comparison.}
  \label{fig:overhead}
\end{figure}
\subsubsection{Overhead Evaluation}
We compare the overhead of ChronoSpike with state-of-the-art dynamic graph learning methods, including recent spiking models and widely used non-spiking baselines.

\textbf{Parameter Size. }
Figure~\ref{param_size} reports the number of trainable parameters. ChronoSpike maintains a compact, dataset-independent parameter footprint. It contains 105K trainable parameters, which remain constant across graph scales and temporal horizons because parameters are shared across time steps. This behavior matches recent spiking models such as SpikeNet, Dy-SIGN, and Delay-DSGN, which also avoid parameter growth over time. In contrast, attention-based methods such as TGAT, which use multi-head temporal attention, require more parameters, while EvolveGCN, which evolves model parameters with RNNs, has a larger parameter size. Larger parameter budgets increase memory usage and overfitting risk. By combining spike-based computation with adaptive LIF neurons and a lightweight temporal aggregation module, ChronoSpike achieves strong performance with a modest parameter budget, comparable to or smaller than existing spiking baselines.

\textbf{Training Time. }
Figure~\ref{train_time} shows that ChronoSpike achieves competitive training efficiency relative to spiking models, SpikeNet, Dy-SIGN, and Delay-DSGN, and is much faster than recurrent and attention-based baselines JODIE, TGAT, and EvolveGCN. Although training time increases with dataset scale, especially on the large \textit{Patent} benchmark, ChronoSpike remains scalable and converges in a few epochs. We attribute this efficiency to two factors:
(i) Like SpikeNet and Delay-DSGN, ChronoSpike uses inductive neighborhood sampling, spike-based message propagation, and lightweight temporal dynamics, which together contribute to the observed training efficiency.
(ii) Compared with recurrent and attention-based temporal models, ChronoSpike uses lightweight temporal aggregation, reducing computation while preserving temporal modeling capacity. Appendix~\ref{app:complexity} provides a theoretical complexity analysis.

\begin{figure}[!ht]
  \centering
  \subfloat[LR vs.\ dropout\label{fig:sensLrDroput}]{%
  \includegraphics[width=0.49\linewidth]{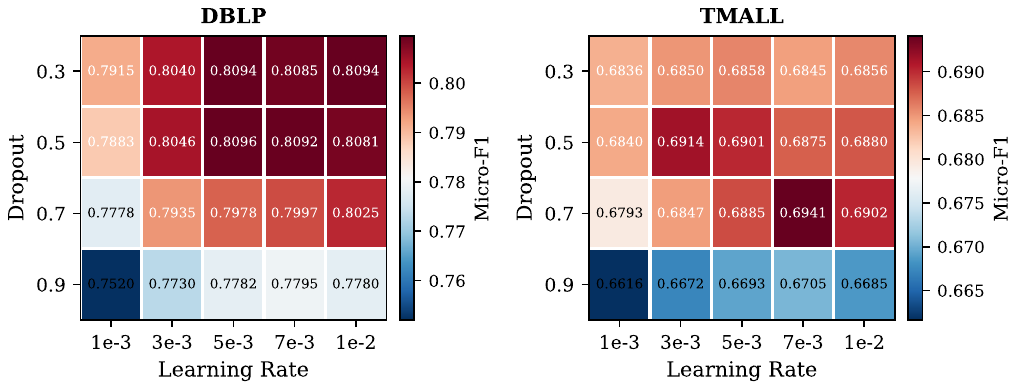}}
  \hfill
  \subfloat[$\alpha$ vs.\ contrastive weight\label{fig:sens_alpha}]{%
  \includegraphics[width=0.49\linewidth]{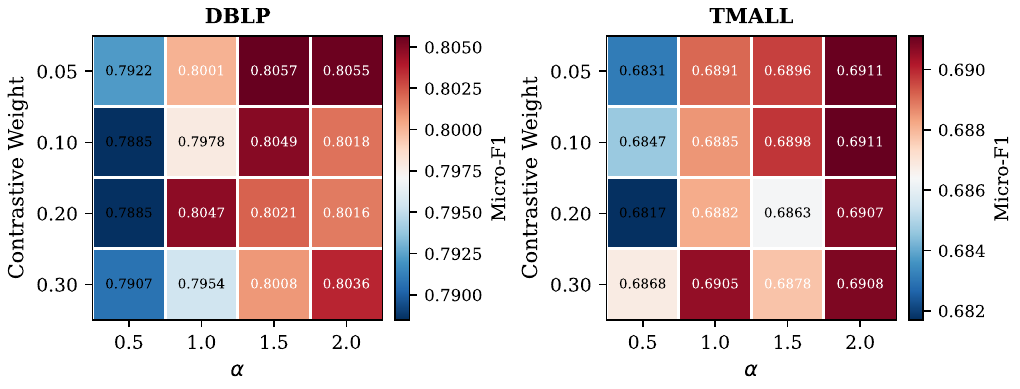}}
  \caption{Parameter sensitivity on DBLP (\textcolor{blue}{left}) and Tmall (\textcolor{myred}{right}) datasets (80\% training). (a) Learning rate vs.\ dropout rate heatmap showing optimal regions. (b) SNN parameter $\alpha$ vs.\ contrastive weight showing robust performance. \textcolor{mygreen}{Darker} colors indicate higher Micro-F1 scores.}
  \label{fig:alphaLrDropout}
\end{figure}

\subsubsection{Parameter Sensitivity Analysis}\label{sec:sensitivity}
We evaluate the robustness and scalability of ChronoSpike through a comprehensive sensitivity analysis. Figures~\ref{fig:alphaLrDropout} and~\ref{fig:samplingBatchHidden} show how the key hyperparameters interact on the DBLP and Tmall datasets, using an 80\% training split. Appendix~\ref{app:hyperparam} details hyperparameter grid search space.

\textbf{Learning rate and dropout. } Figure~\ref{fig:sensLrDroput} shows optimal performance for learning rates $\eta \in [3 \times 10^{-3}, 7 \times 10^{-3}]$ and dropout rates $p_{\text{drop}} \in [0.3, 0.5]$. This broad plateau confirms low sensitivity to these hyperparameters within reasonable ranges. On DBLP, Micro-F1 declines monotonically beyond $p_{\text{drop}}=0.5$ (from $0.8094$ at $0.3$ to $0.7780$ at $0.9$). Tmall is more tolerant: its per-cell maximum (Micro-F1 $=0.6941$) is reached at $p_{\text{drop}}=0.7$, with degradation only at $0.9$. We retain $[0.3, 0.5]$ as the default while noting Tmall's tolerance for stronger regularization.

\textbf{SNN-specific parameters. } Figure~\ref{fig:sens_alpha} plots the interaction between spike steepness ($\alpha$) and contrastive weight. Performance stabilizes for $\alpha\!~\!\in\!~\!\![1.0, 2.0]$, indicating robust neuromorphic approach. Smaller contrastive weights $\lambda\in[0.05,0.10]$ perform best (DBLP peaks at $\lambda{=}0.05$, $\alpha{=}1.5$, Micro-F1 $=0.8057$; Tmall at $\lambda\in\{0.05,0.10\}$, $\alpha{=}2.0$, Micro-F1 $=0.6911$), consistent with our default of $\lambda=0.1$ (Algorithm~\ref{alg:chronospike-training}, Appendix~\ref{app:implementation}).

\textbf{Architecture parameters. }
Figure~\ref{fig:sens_sampling} shows that higher sampling improves performance, with $p=0.8$ optimal for Tmall and $p=0.6$ for DBLP on Macro-F1; DBLP's Micro-F1 instead peaks at $p=0.2$, so $p=0.6$ is a Macro/Micro compromise. Performance is flat across batch sizes $\{256, 500, 1024, 2048\}$ (within $\sim 0.5\%$ F1 on both datasets, Figure~\ref{sens_batch}); we default to $1024$ for throughput. Hidden dimensions of $(d_1, d_2)=(128, 64)$ provide the \textit{best trade-off}, while larger sizes (e.g., $(512, 256)$) add minimal benefits (Figure~\ref{sens_hidden}).

\begin{figure}[!t]
  \centering

  \subfloat[Sampling probability $p$\label{fig:sens_sampling}]{%
  \includegraphics[width=0.48\linewidth]{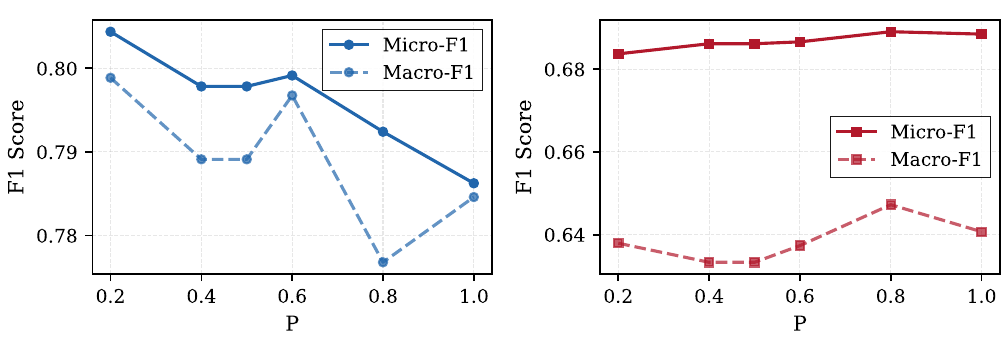}}
  \hfill
  \subfloat[Batch size\label{sens_batch}]{%
  \includegraphics[width=0.48\linewidth]{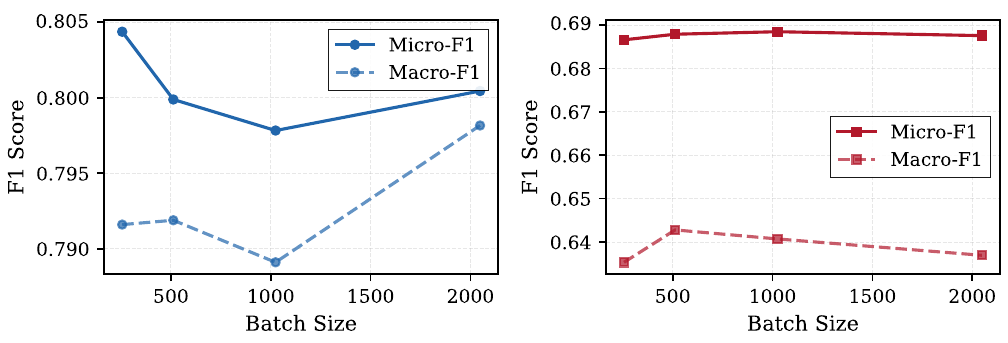}}

  \par\medskip

  \subfloat[Hidden dimension size\label{sens_hidden}]{%
  \includegraphics[width=0.48\linewidth]{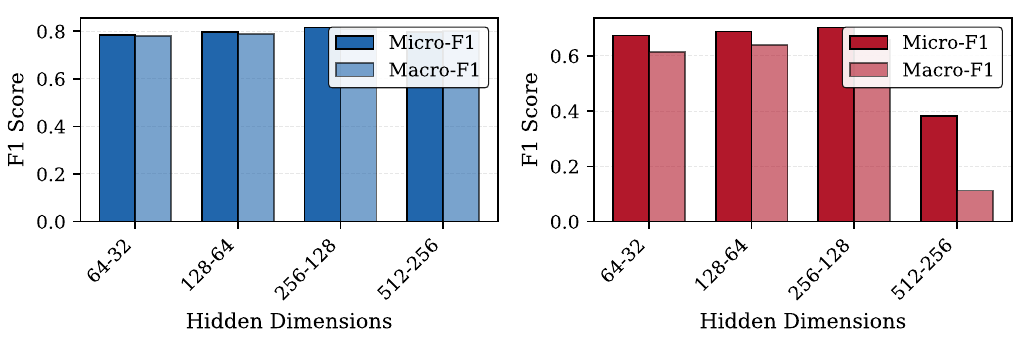}}

  \caption{Sensitivity of ChronoSpike to (a) sampling probability,
    (b) batch size, and (c) hidden dimension size on DBLP
    (\textcolor{blue}{left}) and Tmall (\textcolor{myred}{right})
  datasets using an 80\% training split.}
  \label{fig:samplingBatchHidden}
\end{figure}

\section{Conclusion}
We proposed ChronoSpike, an adaptive spiking graph neural network that combines learnable LIF neurons, multi-head spatial attention, and a Transformer-based temporal encoder for dynamic graph learning. ChronoSpike consistently improves over recent baselines on temporal node classification while keeping memory and parameter cost compact, and our theoretical analysis and interpretability studies show that its spiking dynamics remain stable and learn temporal structure. A natural next step is extending the framework to continuous-time and streaming dynamic graphs, and to deployment on neuromorphic hardware, where its event-driven computation can deliver additional energy savings.

\textbf{Limitations. }
\label{para:limit}
ChronoSpike operates on discrete-time snapshots and requires full historical sequences, limiting its use on irregular, continuous-time, or streaming graphs. Its contrastive module relies on handcrafted augmentations that may not transfer across domains.

{
  \small
  \bibliographystyle{unsrt}
  \bibliography{main}
}

\appendix

\section{Details of Problem Definition}\label{app:problem}
DGRL addresses the task of learning effective node representations from evolving graphs. We consider a dynamic graph that is observed at a sequence of discrete time steps (often referred to as temporal snapshots). Unlike static graphs with fixed structure, a dynamic graph’s topology (edges) and node attributes can change at each snapshot as new connections form or old ones vanish, and node features update to reflect evolving states. Formally, the data consists of a sequence of graph snapshots over $T$ time steps, capturing the network's state at times $1, 2, \dots, T$. This sequential graph input provides the context for modeling how each node’s local neighborhood and features develop over time.

The goal of DGRL is to learn a time-dependent embedding for each node that encodes both the graph's structural information at a given snapshot and the temporal patterns of its evolution up to that point. In other words, for each node $v$ and time $t$, we seek a representation (a $d'$-dimensional vector) that summarizes $v$’s position in the graph at time $t$ as well as its historical interactions. These node embeddings serve as inputs to downstream predictive tasks. In this work, we focus on the temporal node classification problem: given an evolving graph and ground-truth node labels at certain time steps, the objective is to predict each node's label at one or more future or unseen time points. The learned dynamic embeddings should thus be discriminative with respect to node labels, meaning that a simple classifier applied to the embeddings can accurately infer each node’s category at the corresponding time. By framing the problem in a model-agnostic way, we emphasize that any suitable learning method should capture who a node is connected to and when those connections occur, so that the resulting representations reflect the node’s changing context. Success in this setting is typically measured by classification accuracy (or related metrics) on node labels over time, demonstrating that the model has effectively encoded the temporal graph information into the node embeddings.

\section{Details of ChronoSpike}

\subsection{Details of Problem Setting and Notation}\label{app:notation}
We now formalize the dynamic graph learning setup and introduce notation, as adopted by the proposed ChronoSpike framework. We represent a dynamic graph as a sequence of snapshots $G = {G_1, G_2, \dots, G_T}$ over $T$ discrete time steps. Each snapshot $G_t = (V, E_t, X_t)$ consists of a fixed node set $V$ (with $|V|=N$ nodes), an edge set $E_t$ defining the graph connectivity at time $t$, and a matrix of node features $X_t \in \mathbb{R}^{N \times d}$ capturing $d$-dimensional attribute vectors for all nodes. We assume, for simplicity, that the set of nodes $V$ remains consistent over time; only the relationships and features evolve. For any node $v \in V$, we denote its feature vector at time $t$ as $\mathbf{x}_v^{(t)}$ (the $v$th row of $X_t$). Likewise, let $\mathcal{N}_t(v) = \{u \in V : (u,v) \in E_t\}$ denote the set of neighbors of node $v$ in the snapshot $G_t$. This snapshot-based formulation of the dynamic graph implicitly allows new edges (and even new nodes, if they were to appear) to be incorporated at each time step, while tracking the historical trajectory of each node through its changing feature $\mathbf{x}_v^{(t)}$ and neighborhood $\mathcal{N}_t(v)$.

Given this setup, the objective of DGRL can be described as learning a mapping $\Phi_\Theta$ that encodes the sequence of graph snapshots into low-dimensional node embeddings. Concretely, we seek a learnable function
$\Phi_\Theta : (X_1, E_1, \dots, X_T, E_T) \mapsto Z$,
where $Z \in \mathbb{R}^{N \times d'}$ contains a $d'$-dimensional embedding for each of the $N$ nodes (typically $d' \ll d$ for efficiency). The mapping $\Phi_\Theta$ should be designed such that each node’s embedding in $Z$ captures both the structural information from the graph topology and the temporal information from the graph’s evolution up to time $T$. These embeddings can then be fed into a simple classifier or other task-specific module to perform temporal node classification (or other tasks) at the appropriate time step. This formulation defines the learning objective and notation used throughout the remainder of the paper. So, the problem setting provides the foundation: a sequence of evolving graph snapshots and the requirement to embed nodes for classification over time. The notation introduced here will be used throughout the following sections.

\subsection{Details of Dynamic Graph Sampling and Spatial Encoding}\label{app:spatial_encoding}
ChronoSpike employs an inductive neighborhood sampling strategy at each time step to ensure computational tractability on large-scale dynamic graphs. For a given target node $v$ and timestamp $t$, we construct a sampled neighborhood by drawing a fixed number of neighbors from a probabilistic mixture of two graph sources: the cumulative historical graph $G_{\text{cum}} = (V, \bigcup_{t'=1}^{t-1} E_{t'}, \emptyset)$ that aggregates all edges observed up to time $t-1$, and the current snapshot graph $G_t = (V, E_t, X_t)$ containing only edges active at time $t$. Specifically, for each sampled neighbor, we select from $G_{\text{cum}}$ with probability $p$ and from $G_t$ with probability $1-p$, where $p \in [0, 1]$ is a hyperparameter controlling the balance between historical context and recent dynamics. This dual sampling mechanism addresses a fundamental trade-off in dynamic graph learning: purely historical sampling ($p = 1$) captures long-term structural patterns but may miss recent changes, while purely temporal sampling ($p = 0$) emphasizes current interactions but lacks historical context. The hybrid strategy with $p \in (0, 1)$ adaptively combines both perspectives, enabling the model to leverage persistent structural regularities while remaining sensitive to temporal evolution.

Given the sampled neighborhood $\tilde{\mathcal{N}}_t(v)$ of fixed size $S$ for node $v$ at time $t$, ChronoSpike performs spatial feature aggregation through a multi-head attention mechanism that operates on continuous-valued node features before spike generation. This design choice mitigates information loss before spike quantization: by computing attention weights over real-valued features rather than binary spikes, the model retains feature magnitude and relational information during the aggregation phase before conversion into spike representations. For each attention head $h \in \{1, \ldots, H\}$, the per-head spatial representation $\mathbf{h}_{v,h}^{(t)} \in \mathbb{R}^{d_h}$ is computed as:

\begin{equation}\label{eq:8}
\mathbf{h}_{v,h}^{(t)}
=
\big(\mathbf{W}_s\mathbf{x}_v^{(t)}\big)_{(h)}
+
\sum_{u\in\tilde{\mathcal N}_t(v)}
\alpha_{vu}^{(t,h)}
\mathbf{W}_n^{(h)}\mathbf{x}_u^{(t)} 
\end{equation}

where $\mathbf{W}_s \in \mathbb{R}^{d \times d_h}$ is a learnable self-transformation matrix, $\mathbf{W}_n^{(h)} \in \mathbb{R}^{d \times d_h}$ is a head-specific neighbor transformation matrix with $d_h = d/H$ being the per-head dimension, and $\alpha_{vu}^{(t,h)}$ are normalized attention coefficients. The attention weights follow the GAT-style additive parameterization~\cite{velickovic2018graph}:
\begin{equation}\label{eq:9}
  \alpha_{vu}^{(t,h)} = \frac{\exp\!\left(\mathrm{LeakyReLU}\!\left((\mathbf{a}_l^{(h)})^{\!\top}(\mathbf{W}_s\mathbf{x}_v^{(t)})_{(h)} + (\mathbf{a}_r^{(h)})^{\!\top}\mathbf{W}_n^{(h)}\mathbf{x}_u^{(t)}\right)\right)}{\sum_{u'\in\tilde{\mathcal{N}}_t(v)}\exp\!\left(\mathrm{LeakyReLU}\!\left((\mathbf{a}_l^{(h)})^{\!\top}(\mathbf{W}_s\mathbf{x}_v^{(t)})_{(h)} + (\mathbf{a}_r^{(h)})^{\!\top}\mathbf{W}_n^{(h)}\mathbf{x}_{u'}^{(t)}\right)\right)}
\end{equation}
where $\mathbf{a}_l^{(h)}, \mathbf{a}_r^{(h)} \in \mathbb{R}^{d_h}$ are per-head learnable attention vectors and $(\mathbf{W}_s\mathbf{x})_{(h)}\in\mathbb{R}^{d_h}$ denotes the $h$-th per-head slice of the shared self-projection (the implementation uses a single shared $\mathbf{W}_s$ whose output is reshaped to $H$ per-head slices, equivalent to $H$ separate $\mathbf{W}_s^{(h)}$). The per-head outputs are concatenated along the channel dimension to form $\mathbf{h}_v^{(t)} = [\mathbf{h}_{v,1}^{(t)} \,\Vert\, \cdots \,\Vert\, \mathbf{h}_{v,H}^{(t)}] \in \mathbb{R}^{d}$ (with $d=H\,d_h$); we use the concatenation directly without an additional output projection, saving $(H d_h)^2$ parameters per layer. This attentive aggregation mechanism adaptively weights the contributions of different neighbors based on their feature similarity to the target node, enabling the model to focus on the most relevant structural context under evolving graph topologies while maintaining computational efficiency through fixed-size neighborhood sampling.

\subsection{Details of Spike Encoding via Adaptive LIF Neurons}\label{app:adaptive_lif}
ChronoSpike converts aggregated spatial representations into sparse, event-driven spike trains using adaptive LIF neurons. Unlike conventional LIF formulations with fixed time constants and thresholds, our adaptive variant introduces learnable per-channel parameters that enable heterogeneous temporal dynamics across feature dimensions. We follow the line of work on learnable membrane time constants and adaptive thresholds initiated by Fang et al.~\cite{fangIncorporatingLearnableMembrane2021} and extended in~\cite{huangARLIFAdaptiveReset2025,zhangDALIFDualAdaptive2025,yanTrainingHighPerformance2025,yangMMTSNNMarkovianDecision2026}; our contribution is not the introduction of these per-neuron adaptive mechanisms in isolation, but rather their per-channel parameterization within a graph-temporal architecture that combines them with continuous-feature spatial attention and a Transformer-based temporal aggregator (Sections~\ref{subsec:sampling}--\ref{subsec:temporal}). For each node $v$ and feature dimension $i \in \{1, \ldots, d_h\}$, the membrane potential $u_{v,i}^{(t)}$ evolves according to a discrete-time first-order linear difference equation derived from the continuous LIF differential equation:
\begin{equation}\label{eq:10}
  u_{v,i}^{(t)} = u_{v,i}^{(t-1)} + \frac{1}{\tau_i} \left( h_{v,i}^{(t)} - (u_{v,i}^{(t-1)} - u_{\text{reset}}) \right)
\end{equation}
where $h_{v,i}^{(t)}$ is the $i$-th component of the aggregated spatial input from Equation (1), $\tau_i > 0$ is a learnable membrane time constant that controls the temporal integration window for dimension $i$, and $u_{\text{reset}}$ is the reset potential (fixed at zero in our implementation). The time constant $\tau_i$ determines the rate at which the membrane potential decays toward the resting state: larger values of $\tau_i$ result in slower decay and longer temporal integration, allowing the neuron to accumulate inputs over extended periods, while smaller values yield faster responses to transient inputs. This update equation can be rewritten in recursive form as $u_{v,i}^{(t)} = \lambda_i u_{v,i}^{(t-1)} + \frac{1}{\tau_i}(h_{v,i}^{(t)} + u_{\text{reset}})$ where $\lambda_i = 1 - \frac{1}{\tau_i}$ is the decay factor. The stability condition $|\lambda_i| < 1$ requires $\tau_i > \frac{1}{2}$ for our forward-Euler discretization with unit time step, which holds empirically throughout training because $\tau_i$ is initialized to $1.0$ and AdamW updates do not push it outside this region in our runs (see Theorem~\ref{thm:gradient-stability} for a full discussion; the released implementation does not enforce this via a hard reparameterization, so for stricter settings a softplus reparameterization $\tau_i=\tfrac{1}{2}+\mathrm{softplus}(\tilde\tau_i)$ is a one-line drop-in).

A spike is emitted when the membrane potential crosses a learnable firing threshold $V_{\text{th},i}$:
\begin{equation}\label{eq:11}
  s_{v,i}^{(t)} = \mathbb{I}(u_{v,i}^{(t)} \geq V_{\text{th},i}) =
  \begin{cases} 1 & \text{if } u_{v,i}^{(t)} \geq V_{\text{th},i} \\ 0 & \text{otherwise}
  \end{cases}
\end{equation}
where $\mathbb{I}(\cdot)$ denotes the indicator function. Immediately after spike emission, the membrane potential undergoes a hard reset to prevent unbounded accumulation:
\begin{equation}\label{eq:12}
  u_{v,i}^{(t)} \leftarrow (1 - s_{v,i}^{(t)}) u_{v,i}^{(t)} + s_{v,i}^{(t)} u_{\text{reset}}
\end{equation}
Following standard SNN terminology~\cite{fangIncorporatingLearnableMembrane2021}, this is a hard reset (reset-to-fixed-value): upon firing, the potential is forced to $u_{\text{reset}}$ rather than reduced by $V_{\text{th},i}$ as in a soft (subtractive) reset.

This reset mechanism ensures that the neuron returns to a bounded state after firing, which we prove maintains numerical stability in Theorem \ref{thm:lif-bounded}. The learnable parameters $\{\tau_i, V_{\text{th},i}\}_{i=1}^{d_h}$ are optimized jointly with all other model parameters via backpropagation, enabling the network to discover task-specific temporal dynamics. Different feature channels can learn different integration time scales and firing rates, resulting in heterogeneous temporal receptive fields that improve the model's ability to capture diverse temporal patterns.

The primary challenge in training spiking neural networks is the non-differentiability of the spike generation function $s_{v,i}^{(t)} = \mathbb{I}(u_{v,i}^{(t)} \geq V_{\text{th},i})$. The indicator function has a zero gradient almost everywhere and an undefined gradient at the threshold, preventing standard backpropagation. ChronoSpike addresses this using the surrogate gradient method, where we approximate the derivative during the backward pass while preserving the true discrete spike function in the forward pass. Specifically, we replace the gradient of $\mathbb{I}(\cdot)$ with the derivative of a smooth sigmoid-like function:
\begin{equation}
  \frac{\partial s_{v,i}^{(t)}}{\partial u_{v,i}^{(t)}} \approx \sigma'(u_{v,i}^{(t)} - V_{\text{th},i}; \alpha) = \frac{\alpha}{(\alpha |u_{v,i}^{(t)} - V_{\text{th},i}| + 1)^2}
\end{equation}

where $\alpha > 0$ is a slope parameter controlling the width of the gradient approximation (initialized to $\alpha=1.0$ and learned per feature channel during training, so $\alpha_i$ in fact varies across channels rather than being a fixed global hyperparameter). This surrogate gradient provides a continuous, differentiable approximation to the Heaviside step function's derivative, enabling gradient flow through the spike generation process while maintaining the computational and representational benefits of binary spikes during forward propagation.

\subsection{Details of Temporal Spike Integration}\label{app:temporal_integration}
ChronoSpike aggregates temporal information across the entire sequence of graph snapshots by processing the collected spike representations through a lightweight Transformer-based encoder. For each node $v$, we stack the spike outputs $\{\mathbf{s}_v^{(1)}, \mathbf{s}_v^{(2)}, \ldots, \mathbf{s}_v^{(T)}\}$ from all $T$ time steps into a temporal sequence matrix $\mathbf{S}_v \in \mathbb{R}^{T \times d_h}$, where each row $\mathbf{s}_v^{(t)} \in \mathbb{R}^{d_h}$ represents the spike-encoded feature vector at time $t$. To preserve temporal order information, which is critical for capturing causal dependencies in dynamic graphs, we augment this sequence with learnable positional encodings $\mathbf{P}_{1:T} \in \mathbb{R}^{T \times d_h}$ that provide each time step with a unique positional identifier:
\begin{equation}
  \tilde{\mathbf{S}}_v = \mathbf{S}_v + \mathbf{P}_{1:T}
\end{equation}
where addition is performed element-wise. The positional encodings $\mathbf{P}_{1:T}$ are initialized randomly and optimized during training, allowing the model to learn task-specific temporal representations rather than relying on fixed sinusoidal or absolute position encodings.

The positionally-encoded spike sequence $\tilde{\mathbf{S}}_v$ is then processed by a single-layer Transformer encoder with multi-head self-attention:
\begin{align}
  \mathbf{Y}_v &= \text{LayerNorm}\big(\tilde{\mathbf{S}}_v + \text{MultiHeadAttn}(\tilde{\mathbf{S}}_v)\big), \\
  \mathbf{Z}_v^{\text{temp}} &= \text{TransformerEncoder}(\tilde{\mathbf{S}}_v) = \text{LayerNorm}\big(\mathbf{Y}_v + \text{FFN}(\mathbf{Y}_v)\big)
\end{align}

where the multi-head self-attention mechanism computes pairwise interactions across all time steps via scaled dot-product attention, the feed-forward network (FFN) applies two linear transformations with ReLU activations, and LayerNorm is applied with the standard pre-/post-residual placement of a Transformer encoder block, providing both the attention-side and FFN-side residual connections. The self-attention mechanism enables each time step representation to attend to all other time steps in the sequence, capturing long-range temporal dependencies that cannot be modeled by sequential recurrence or local temporal convolutions. The attention weights learned by this module reveal which historical snapshots are most relevant for predicting node behavior at the final time step, as we demonstrate through interpretability analysis in Appendix G.2.

The final temporal embedding $\mathbf{z}_v \in \mathbb{R}^{d_h}$ for node $v$ is extracted from the last time step of the Transformer output and then passed through a small readout head: $\mathbf{z}_v = \mathbf{W}_2\,\tanh(\mathbf{W}_1\,\mathbf{Z}_v^{\text{temp}}[T,:] + \mathbf{b}_1) + \mathbf{b}_2$ (a two-layer MLP with $\mathbf{W}_1,\mathbf{W}_2\in\mathbb{R}^{d_h\times d_h}$). This readout, omitted from earlier symbolic equations for compactness, accounts for an additional $\approx 2 d_h^2 + 2 d_h$ trainable parameters and matches the released implementation, which aggregates information from the entire temporal sequence through the self-attention mechanism. This design decouples temporal modeling from recurrent state propagation, avoiding the sequential dependency and gradient pathology issues that plague RNN-based approaches while maintaining linear memory complexity in the temporal dimension, since we only store spike representations rather than dense hidden states at each time step.

\subsection{Details of Joint Optimization with Contrastive Regularization}\label{app:optimization}
ChronoSpike is trained end-to-end using a composite objective function that combines supervised classification loss with temporal contrastive regularization. For the classification task, we apply a linear transformation followed by softmax normalization to the final node embeddings to produce class probability distributions:

$$\hat{\mathbf{y}}_v = \text{Softmax}(\mathbf{W}_c \mathbf{z}_v + \mathbf{b}_c) = \frac{\exp(\mathbf{W}_c^{(c)} \mathbf{z}_v + b_c^{(c)})}{\sum_{c'=1}^C \exp(\mathbf{W}_c^{(c')} \mathbf{z}_v + b_{c'}^{(c')})}$$

where $\mathbf{W}_c \in \mathbb{R}^{C \times d_h}$ and $\mathbf{b}_c \in \mathbb{R}^C$ are learnable classifier parameters, $C$ is the number of classes, and $\mathbf{W}_c^{(c)}$ denotes the $c$-th row of $\mathbf{W}_c$. The supervised classification loss is computed as the cross-entropy between predicted and true label distributions over the set of labeled training nodes $\mathcal{V}_{\text{train}} \subseteq V$:

$$\mathcal{L}_{\text{cls}} = -\frac{1}{|\mathcal{V}_{\text{train}}|} \sum_{v \in \mathcal{V}_{\text{train}}} \sum_{c=1}^C y_v^{(c)} \log \hat{y}_v^{(c)} = -\frac{1}{|\mathcal{V}_{\text{train}}|} \sum_{v \in \mathcal{V}_{\text{train}}} \log \hat{y}_v^{(y_v)}$$

where $y_v \in \{1, \ldots, C\}$ is the true class label for node $v$, and we use the simplified form for one-hot encoded labels.

To enhance the robustness and generalization of learned representations, ChronoSpike incorporates a temporal contrastive regularization term inspired by SimCLR. For each node embedding $\mathbf{z}_v$, we generate two stochastic augmented views $\mathbf{z}_v^{(1)}$ and $\mathbf{z}_v^{(2)}$ by applying independent feature dropout with rate $p_{\text{con}} = 0.1$. The contrastive objective encourages these two views of the same node to have similar representations while pushing apart representations of different nodes:

$$\mathcal{L}_{\text{con}} = -\frac{1}{|\mathcal{V}_{\text{train}}|} \sum_{v \in \mathcal{V}_{\text{train}}} \log \frac{\exp(\text{sim}(\mathbf{z}_v^{(1)}, \mathbf{z}_v^{(2)}) / \tau)}{\sum_{u \in \mathcal{V}_{\text{train}}} \exp(\text{sim}(\mathbf{z}_v^{(1)}, \mathbf{z}_u^{(2)}) / \tau)}$$

where $\text{sim}(\mathbf{a}, \mathbf{b}) = \frac{\mathbf{a}^\top \mathbf{b}}{\|\mathbf{a}\|_2 \|\mathbf{b}\|_2}$ is the cosine similarity between vectors $\mathbf{a}$ and $\mathbf{b}$, and $\tau > 0$ is a temperature hyperparameter controlling the concentration of the distribution (we use $\tau = 0.5$). This contrastive term acts as a regularizer that improves representation quality and prevents overfitting, particularly under limited supervision. The final training objective combines both loss components:
\begin{equation}
  \mathcal{L} = \mathcal{L}_{\text{cls}} + \lambda \mathcal{L}_{\text{con}}
\end{equation}
where $\lambda = 0.1$ is a weighting coefficient. The contrastive loss is only applied during training and is disabled during inference to reduce computational cost.

Optimization is performed using the AdamW optimizer with mini-batch stochastic gradient descent. At each iteration, we sample a batch of target nodes, construct their temporal neighborhoods through the sampling procedure described in Section \ref{subsec:sampling}, perform forward propagation through spatial aggregation, spike generation, and temporal integration, and compute gradients of the composite loss with respect to all model parameters. The surrogate gradient approximation for spike generation is applied only during backpropagation, while the forward pass uses true binary spikes, ensuring that the trained model maintains spike-based computation and neuromorphic compatibility while remaining trainable through standard gradient descent.

\begin{algorithm}[!ht]
  \scriptsize
  \caption{ChronoSpike Encoder -- Forward Computation}
  \label{alg:chronospike-encoder}
  \begin{algorithmic}[1]
    \REQUIRE Dynamic graph snapshots $\{G_1, \ldots, G_T\}$ where $G_t = (V, E_t, X_t)$
    \REQUIRE Target node batch $\mathcal{V}_b \subseteq V$, $|\mathcal{V}_b| = B$
    \REQUIRE Layer-wise neighborhood sizes $\{S_1, \ldots, S_K\}$
    \REQUIRE Sampling probability $p \in [0,1]$ for hybrid sampling
    \REQUIRE Model parameters $\Theta = \{\mathbf{W}_s^{(k)}, \mathbf{W}_n^{(k)}, \mathbf{W}_q^{(k)}, \mathbf{W}_k^{(k)}, \mathbf{W}_v^{(k)}\}_{k=1}^K \cup \{\tau_i, V_{\text{th},i}\}_{i=1}^d$
    \ENSURE Temporal node embeddings $\mathbf{Z} \in \mathbb{R}^{B \times d}$
    \STATE \textcolor{blue}{\textbf{// Stage 1: Initialize data structures}}
    \STATE Initialize spike sequence list $\mathcal{S} \leftarrow []$ \hfill $\triangleright$ Store spike outputs across time
    \STATE Initialize per-layer membrane potentials $\mathbf{u}_k^{(0)} \leftarrow \mathbf{0} \in \mathbb{R}^{B \times d}$ for $k=1,\ldots,K$ \hfill $\triangleright$ One state per spatial layer; persists across timesteps
    \STATE Initialize cumulative graph $G_{\text{cum}} = (V, \emptyset, \emptyset)$ \hfill $\triangleright$ Historical structure (updated after each timestep)

    \STATE \textcolor{blue}{\textbf{// Stage 2: Temporal spike generation loop}}
    \FOR{$t = 1$ to $T$}
    \STATE \textcolor{teal}{\textbf{// 2.1: Load current snapshot}}
    \STATE $G_t \leftarrow (V, E_t, X_t)$ \hfill $\triangleright$ Current graph state at timestep $t$

    \STATE \textcolor{teal}{\textbf{// 2.2: Hybrid neighborhood sampling}}
    \STATE $G_{\text{mix}}^{(t)} \leftarrow \text{ProbabilisticMix}(G_{\text{cum}}, G_t, p)$ \hfill $\triangleright$ With prob. $p$: sample from $G_{\text{cum}}$; else: $G_t$

    \STATE \textcolor{teal}{\textbf{// 2.3: Initialize node features for spatial propagation}}
    \STATE Build the multi-hop computational tree: starting from $\mathcal{V}_b$, iteratively sample $\mathcal{N}_k(v)$ for $k=1,\ldots,K$ to obtain the union $\tilde{\mathcal{V}}=\mathcal{V}_b\cup\bigcup_k\mathcal{N}_k$
    \STATE $\mathbf{H}_0 \leftarrow X_t[\tilde{\mathcal{V}}, :]$ \hfill $\triangleright$ Feature matrix for all nodes in the computation tree, $\in \mathbb{R}^{|\tilde{\mathcal{V}}| \times d_{\text{in}}}$
    \STATE $\mathcal{V}_0 \leftarrow \mathcal{V}_b$ \hfill $\triangleright$ Current active node set

    \STATE \textcolor{teal}{\textbf{// 2.4: K-layer spatial aggregation with attention}}
    \FOR{$k = 1$ to $K$}
    \STATE \textcolor{violet}{\textbf{// 2.4.1: Sample neighborhoods}}
    \FOR{each node $v \in \mathcal{V}_{k-1}$}
    \STATE $\mathcal{N}_k(v) \leftarrow \text{Sample}(S_k \text{ nodes from } \mathcal{N}_{G_{\text{mix}}^{(t)}}(v))$
    \ENDFOR
    \STATE $\mathcal{V}_k \leftarrow \bigcup_{v \in \mathcal{V}_{k-1}} \mathcal{N}_k(v)$ \hfill $\triangleright$ Expanded node set

    \STATE \textcolor{violet}{\textbf{// 2.4.2: Multi-head attention aggregation}}
    \STATE $\mathbf{Q} \leftarrow \mathbf{H}_{k-1}[\mathcal{V}_{k-1},:] \mathbf{W}_q^{(k)}$ \hfill $\triangleright$ Query $\in \mathbb{R}^{|\mathcal{V}_{k-1}| \times d_h}$
    \STATE $\mathbf{K} \leftarrow \mathbf{H}_{k-1}[\mathcal{N}_k, :] \mathbf{W}_k^{(k)}$ \hfill $\triangleright$ Keys for sampled neighbors (well-defined because $\tilde{\mathcal{V}}\supseteq\mathcal{N}_k$)
    \STATE $\mathbf{V} \leftarrow \mathbf{H}_{k-1}[\mathcal{N}_k, :] \mathbf{W}_v^{(k)}$ \hfill $\triangleright$ Neighbor values
    \STATE $\mathbf{A} \leftarrow \text{GATAttention}(\mathbf{H}_{k-1}[\mathcal{V}_{k-1},:], \mathbf{H}_{k-1}[\mathcal{N}_k,:])$ \hfill $\triangleright$ Additive attention over sampled neighbors
    \STATE $\mathbf{H}_{\text{agg}} \leftarrow \mathbf{A} \mathbf{V}$ \hfill $\triangleright$ Weighted aggregation

    \STATE \textcolor{violet}{\textbf{// 2.4.3: Combine self and neighbor features}}
    \STATE $\mathbf{H}_k^{\text{temp}} \leftarrow \mathbf{H}_{k-1}[\mathcal{V}_{k-1},:] \mathbf{W}_s^{(k)} + \mathbf{H}_{\text{agg}} \mathbf{W}_n^{(k)}$ \hfill $\triangleright$ Element-wise addition

    \STATE \textcolor{violet}{\textbf{// 2.4.4: Generate spikes via adaptive LIF neurons (per-layer state $\mathbf{u}_k$})}
    \FOR{each feature dimension $i = 1$ to $d$}
    \STATE $u_{k,i}^{(t)} \leftarrow u_{k,i}^{(t-1)} + \frac{1}{\tau_i}\left(H_k^{\text{temp}}[:, i] - (u_{k,i}^{(t-1)} - u_{\text{reset}})\right)$ \hfill $\triangleright$ Layer-$k$ membrane update (Eq. \ref{eq:10})
    \STATE $s_{k,i}^{(t)} \leftarrow \mathbb{I}(u_{k,i}^{(t)} \geq V_{\text{th},i})$ \hfill $\triangleright$ Spike generation (Eq. \ref{eq:11}): 1 if firing, 0 otherwise
    \STATE $u_{k,i}^{(t)} \leftarrow (1 - s_{k,i}^{(t)}) \cdot u_{k,i}^{(t)} + s_{k,i}^{(t)} \cdot u_{\text{reset}}$ \hfill $\triangleright$ Reset membrane (Eq. \ref{eq:12})
    \ENDFOR

    \STATE $\mathbf{H}_k \leftarrow \mathbf{s}_k^{(t)}$ \hfill $\triangleright$ Spike outputs become next layer input
    \IF{$k < K$}
    \STATE $\mathbf{H}_k \leftarrow \text{Dropout}(\mathbf{H}_k, p=0.7)$ \hfill $\triangleright$ Regularization between layers
    \ENDIF
    \ENDFOR

    \STATE \textcolor{teal}{\textbf{// 2.5: Store final-layer spikes for this timestep}}
    \STATE Append $\mathbf{s}^{(t)} \leftarrow \mathbf{H}_K$ to spike sequence $\mathcal{S}$
    \STATE \textcolor{teal}{\textbf{// 2.6: Membrane state $\{\mathbf{u}_k^{(t)}\}_{k=1}^{K}$ persists for use at timestep $t{+}1$ (no within-loop reset).}}
    \STATE \textcolor{teal}{\textbf{// 2.7: Update cumulative graph}}
    \STATE $E_{\text{cum}} \leftarrow E_{\text{cum}} \cup E_t$ \hfill $\triangleright$ Incrementally add edges
    \ENDFOR

    \STATE \textcolor{blue}{\textbf{// Stage 3: Temporal aggregation via Transformer}}
    \STATE $\mathbf{S} \leftarrow \text{Stack}(\mathcal{S}) \in \mathbb{R}^{B \times T \times d}$ \hfill $\triangleright$ Stack spike sequence
    \STATE $\mathbf{S}_{\text{pos}} \leftarrow \mathbf{S} + \mathbf{P}_{1:T}$ \hfill $\triangleright$ Add learned positional encodings
    \STATE $\mathbf{Z}_{\text{temp}} \leftarrow \text{TransformerEncoder}(\mathbf{S}_{\text{pos}})$ \hfill $\triangleright$ Apply Transformer encoder
    \STATE $\mathbf{Z} \leftarrow \mathbf{Z}_{\text{temp}}[:, -1, :]$ \hfill $\triangleright$ Extract last timestep representation

    \STATE \textbf{return} $\mathbf{Z} \in \mathbb{R}^{B \times d}$ \hfill $\triangleright$ Node embeddings for downstream tasks
  \end{algorithmic}
\end{algorithm}

\begin{algorithm}[!ht]
  \scriptsize
  \caption{ChronoSpike Training Loop}
  \label{alg:chronospike-training}
  \begin{algorithmic}[1]
    \REQUIRE Training node set $\mathcal{V}_{\text{train}} \subseteq V$ with labels $\mathbf{y} \in \{1, \ldots, C\}^{|\mathcal{V}_{\text{train}}|}$
    \REQUIRE Validation node set $\mathcal{V}_{\text{val}} \subseteq V$ with labels $\mathbf{y}_{\text{val}}$
    \REQUIRE Batch size $B = 1024$, learning rate $\eta$, number of epochs $E = 100$
    \REQUIRE Contrastive weight $\lambda = 0.1$, temperature $\tau_{\text{con}} = 0.5$, dropout rates $p_{\text{drop}} = 0.7$, $p_{\text{con}} = 0.1$
    \ENSURE Trained model parameters $\Theta^*$

    \STATE \textcolor{blue}{\textbf{// Initialization}}
    \STATE Initialize model parameters $\Theta$ (Xavier/Glorot initialization for weights)
    \STATE Initialize AdamW optimizer with learning rate $\eta$, weight decay $\lambda_{\text{wd}} = 1 \times 10^{-4}$
    \STATE Partition $\mathcal{V}_{\text{train}}$ into mini-batches $\{\mathcal{V}_b^{(1)}, \ldots, \mathcal{V}_b^{(M)}\}$ where $|\mathcal{V}_b^{(j)}| = B$
    \STATE Initialize best validation performance $\text{BestVal} \leftarrow 0$, patience counter $\text{Patience} \leftarrow 0$

    \STATE
    \FOR{epoch $= 1$ to $E$}
    \STATE \textcolor{teal}{\textbf{// Training phase}}
    \STATE Set model to training mode (enable dropout, batch normalization)
    \STATE Shuffle mini-batch order $\{\mathcal{V}_b^{(1)}, \ldots, \mathcal{V}_b^{(M)}\}$

    \FOR{each mini-batch $\mathcal{V}_b$ in shuffled order}
    \STATE \textcolor{violet}{\textbf{// Forward pass: compute embeddings and logits}}
    \STATE $\mathbf{Z}_b \leftarrow \textsc{ChronoSpikeEncoder}(\mathcal{V}_b; \Theta)$ \COMMENT{Algo 1, returns $\mathbf{Z}_b \in \mathbb{R}^{B \times d}$}
    \STATE $\hat{\mathbf{y}}_b \leftarrow \text{Softmax}(\mathbf{Z}_b \mathbf{W}_c + \mathbf{b}_c)$ \COMMENT{Linear classifier, $\mathbf{W}_c \in \mathbb{R}^{d \times C}$}

    \STATE \textcolor{violet}{\textbf{// Compute classification loss (cross-entropy)}}
    \STATE $\mathcal{L}_{\text{cls}} \leftarrow -\frac{1}{B} \sum_{v \in \mathcal{V}_b} \log \hat{y}_v[y_v]$ \COMMENT{Negative log-likelihood}

    \STATE \textcolor{violet}{\textbf{// Generate stochastic views for contrastive learning}}
    \STATE $\mathbf{Z}_b^{(1)} \leftarrow \text{Dropout}(\mathbf{Z}_b, p_{\text{con}})$ \COMMENT{First view with 10\% dropout}
    \STATE $\mathbf{Z}_b^{(2)} \leftarrow \text{Dropout}(\mathbf{Z}_b, p_{\text{con}})$ \COMMENT{Second view, independently sampled}

    \STATE \textcolor{violet}{\textbf{// Compute temporal contrastive loss (InfoNCE)}}
    \STATE Normalize embeddings: $\tilde{\mathbf{Z}}_b^{(1)} \leftarrow \mathbf{Z}_b^{(1)} / \|\mathbf{Z}_b^{(1)}\|_2$, $\tilde{\mathbf{Z}}_b^{(2)} \leftarrow \mathbf{Z}_b^{(2)} / \|\mathbf{Z}_b^{(2)}\|_2$
    \STATE Compute similarity matrix: $\mathbf{S} \leftarrow \tilde{\mathbf{Z}}_b^{(1)} (\tilde{\mathbf{Z}}_b^{(2)})^\top / \tau_{\text{con}}$ \COMMENT{$\in \mathbb{R}^{B \times B}$}
    \STATE Compute contrastive loss:
    \begin{equation*}
      \mathcal{L}_{\text{con}} = -\frac{1}{B} \sum_{i=1}^B \log \frac{\exp(S_{ii})}{\sum_{j=1}^B \exp(S_{ij})}
    \end{equation*}
    \STATE \COMMENT{Encourages agreement between views on diagonal, disagreement off-diagonal}

    \STATE \textcolor{violet}{\textbf{// Combine losses}}
    \STATE $\mathcal{L} \leftarrow \mathcal{L}_{\text{cls}} + \lambda \mathcal{L}_{\text{con}}$ \COMMENT{Weighted sum}

    \STATE \textcolor{violet}{\textbf{// Backward pass and parameter update}}
    \STATE Zero gradients: $\nabla_\Theta \leftarrow 0$
    \STATE Compute gradients: $\nabla_\Theta \mathcal{L}$ using surrogate gradients for spike function
    \STATE Clip gradients: $\nabla_\Theta \leftarrow \text{Clip}(\nabla_\Theta, \text{max\_norm}=1.0)$ \COMMENT{Prevent exploding gradients}
    \STATE Update parameters: $\Theta \leftarrow \text{AdamW}(\Theta, \nabla_\Theta \mathcal{L}, \eta, \lambda_{\text{wd}})$
    \ENDFOR

    \STATE
    \STATE \textcolor{teal}{\textbf{// Validation phase}}
    \STATE Set model to evaluation mode (disable dropout, use deterministic forward pass)
    \STATE Compute validation embeddings: $\mathbf{Z}_{\text{val}} \leftarrow \textsc{ChronoSpikeEncoder}(\mathcal{V}_{\text{val}}; \Theta)$
    \STATE Compute validation logits: $\hat{\mathbf{y}}_{\text{val}} \leftarrow \text{Softmax}(\mathbf{Z}_{\text{val}} \mathbf{W}_c + \mathbf{b}_c)$
    \STATE Compute validation Macro-F1 score: $\text{ValF1} \leftarrow \text{MacroF1}(\hat{\mathbf{y}}_{\text{val}}, \mathbf{y}_{\text{val}})$

    \STATE \textcolor{teal}{\textbf{// Early stopping and checkpointing}}
    \IF{$\text{ValF1} > \text{BestVal}$}
    \STATE $\text{BestVal} \leftarrow \text{ValF1}$
    \STATE Save model checkpoint: $\Theta^* \leftarrow \Theta$
    \STATE Reset patience: $\text{Patience} \leftarrow 0$
    \ELSE
    \STATE $\text{Patience} \leftarrow \text{Patience} + 1$
    \ENDIF

    \IF{$\text{Patience} \geq 10$}
    \STATE \textbf{break} \COMMENT{Early stopping after 10 epochs without improvement}
    \ENDIF

    \STATE Log epoch metrics: $\mathcal{L}_{\text{cls}}$, $\mathcal{L}_{\text{con}}$, $\text{ValF1}$
    \ENDFOR

    \STATE
    \STATE \textbf{return} $\Theta^*$ \COMMENT{Best model parameters based on validation performance}
  \end{algorithmic}
\end{algorithm}

\begin{algorithm}[!ht]
  \scriptsize
  \caption{ChronoSpike Inference}
  \label{alg:chronospike-inference}
  \begin{algorithmic}[1]
    \REQUIRE Test node set $\mathcal{V}_{\text{test}} \subseteq V$
    \REQUIRE Trained model parameters $\Theta^*$ (loaded from checkpoint)
    \REQUIRE Inference batch size $B_{\text{inf}}$ (dataset-dependent)
    \ENSURE Predicted class labels $\hat{\mathbf{y}}_{\text{test}} \in \{1, \ldots, C\}^{|\mathcal{V}_{\text{test}}|}$

    \STATE \textcolor{blue}{\textbf{// Initialization}}
    \STATE Set model to evaluation mode (disable dropout, batch normalization uses running stats)
    \STATE Partition $\mathcal{V}_{\text{test}}$ into batches $\{\mathcal{V}_b^{(1)}, \ldots, \mathcal{V}_b^{(N)}\}$ where $|\mathcal{V}_b^{(j)}| \leq B_{\text{inf}}$
    \STATE Initialize prediction list $\hat{\mathbf{y}}_{\text{test}} \leftarrow []$

    \STATE
    \FOR{each test batch $\mathcal{V}_b$ in $\{\mathcal{V}_b^{(1)}, \ldots, \mathcal{V}_b^{(N)}\}$}
    \STATE \textcolor{teal}{\textbf{// Forward pass without gradient tracking}}
    \STATE \textbf{with} torch.no\_grad(): \COMMENT{Disable autograd to save memory}
    \STATE $\mathbf{Z}_b \leftarrow \textsc{ChronoSpikeEncoder}(\mathcal{V}_b; \Theta^*)$ \COMMENT{Deterministic, no dropout}
    \STATE $\hat{\mathbf{y}}_b \leftarrow \text{Softmax}(\mathbf{Z}_b \mathbf{W}_c + \mathbf{b}_c)$ \COMMENT{Class probabilities}
    \STATE $\hat{y}_b^{\text{pred}} \leftarrow \arg\max \hat{\mathbf{y}}_b$ \COMMENT{Predicted class per node}

    \STATE \textcolor{teal}{\textbf{// Accumulate predictions}}
    \STATE Append $\hat{y}_b^{\text{pred}}$ to $\hat{\mathbf{y}}_{\text{test}}$
    \ENDFOR

    \STATE
    \STATE \textbf{return} $\hat{\mathbf{y}}_{\text{test}}$ \COMMENT{Concatenated predictions for all test nodes}
  \end{algorithmic}
\end{algorithm}

\subsection{Detailed Algorithm}
\label{app:algorithm}
This appendix provides a complete, implementation-faithful description of the ChronoSpike framework. We present detailed pseudocode for the model's forward computation, training procedure, and inference protocol. The purpose is to ensure full reproducibility, clarify algorithmic details omitted from the main paper due to space constraints, and formally connect the proposed method to its practical implementation.

\subsubsection{Overview and Computational Flow}
ChronoSpike operates on a sequence of dynamic graph snapshots $\mathcal{G} = \{G_1, G_2, \ldots, G_T\}$ and learns node representations by combining spatial neighborhood aggregation with spike-based temporal modeling. Temporal dependencies are captured through LIF neuron dynamics and sparse spike accumulation, avoiding the storage of dense hidden states across timesteps. As a result, the memory complexity is reduced to $O(|V| \cdot d)$ per batch. The complete ChronoSpike pipeline consists of four major stages:

\begin{enumerate}[leftmargin=*, itemsep=2pt]
  \item \textbf{Dynamic Graph Sampling:} For each target node $v$ at timestep $t$, we perform hybrid neighborhood sampling that combines historical structure (cumulative edges from $G_1, \ldots, G_{t-1}$) with current temporal context (edges from $G_t$). This dual-sampling strategy balances long-term structural memory with recent temporal dynamics.

  \item \textbf{Spatial Feature Encoding:} Sampled neighborhoods undergo $K$-layer graph aggregation using multi-head attention mechanisms. Each layer $k$ applies learnable transformations $\mathbf{W}_s^{(k)}, \mathbf{W}_n^{(k)} \in \mathbb{R}^{d \times d_h}$ to aggregate neighbor features adaptively, producing synaptic inputs $\mathbf{I}^{(t)}$ for the spiking neurons.

  \item \textbf{Spike-Based Temporal Modeling:} Aggregated spatial features are converted into sparse spike trains via adaptive LIF neurons. Each neuron $i$ maintains a membrane potential $u_i^{(t)}$ governed by learnable time constants $\tau_i$ and thresholds $V_{\text{th},i}$, enabling heterogeneous temporal dynamics across feature channels. Spikes $s_i^{(t)} \in \{0,1\}$ are emitted when $u_i^{(t)}$ exceeds $V_{\text{th},i}$, providing event-driven, energy-efficient computation.

  \item \textbf{Temporal Aggregation and Prediction:} Spike representations $\{\mathbf{s}^{(1)}, \ldots, \mathbf{s}^{(T)}\}$ across all timesteps are aggregated using a lightweight Transformer encoder with learned positional encodings. This global temporal aggregation captures long-range dependencies without recurrent state propagation. The final embedding $\mathbf{z}_v$ is passed through a linear classifier to produce class predictions.
\end{enumerate}

% \textbf{Key Algorithmic Innovations:} (i) \emph{Stateless Temporal Modeling:} Unlike RNN-based methods that sequentially update $\mathbf{h}_v^{(t)} = f(\mathbf{h}_v^{(t-1)}, \mathbf{x}_v^{(t)})$, ChronoSpike processes each snapshot independently and integrates temporal information via spike accumulation and Transformer aggregation, enabling parallelization across timesteps. (ii) \emph{Adaptive Spiking Dynamics:} Per-channel learnable parameters $\{\tau_i, V_{\text{th},i}\}$ allow each feature dimension to develop specialized temporal receptive fields, unlike fixed-threshold SNNs. (iii) \emph{Hybrid Sampling Strategy:} The probabilistic mixing of cumulative and current graphs (controlled by parameter $p$) adaptively balances historical context and recent changes, crucial for non-stationary dynamic graphs.

\subsubsection{ChronoSpike Forward Computation}
\label{app:chronospike-forward}
Algorithm~\ref{alg:chronospike-encoder} presents the complete forward pass of the ChronoSpike encoder. Given a batch of target nodes $\mathcal{V}_b$ and the sequence of graph snapshots, the encoder computes temporal embeddings by iterating through each timestep, performing spatial aggregation, generating spikes, and finally aggregating spike sequences temporally.

\subsubsection{Training Procedure: End-to-End Optimization}
\label{app:training-procedure}
The training objective of ChronoSpike combines a supervised classification loss with an auxiliary temporal contrastive regularization term that encourages robustness of the learned representations to stochastic perturbations. Algorithm~\ref{alg:chronospike-training} presents the complete training loop with detailed annotations.

\subsubsection{Inference Procedure: Efficient Deployment}
\label{app:inference-procedure}
At inference time, ChronoSpike operates in a simplified, deterministic mode without gradient computation, dropout, or contrastive loss. Algorithm~\ref{alg:chronospike-inference} describes the inference protocol.

\section{Computational Complexity Analysis}\label{app:complexity}

\textbf{Time Complexity.} ChronoSpike has linear activation/state memory scaling with the temporal horizon under fixed fanout and bounded horizons, with an additional per-node temporal attention cost. Let $N = |V|$ denote the number of nodes, $T$ the number of discrete time steps (graph snapshots), $S$ the number of sampled neighbors per node per layer, $d$ the hidden feature dimensionality (hidden state size), and $K$ the number of graph propagation layers. At each time step, ChronoSpike applies $K$ rounds of neighborhood aggregation. In each round, a node aggregates messages from at most $S$ neighbors, followed by a feature transformation using dense weight matrices of size $d \times d$ (e.g., linear projections in attention or gated aggregation). This results in a per-node, per-layer computational cost of $O(S d^2)$. Aggregating over all nodes and layers yields a per-time-step cost of $O(N \cdot S \cdot K \cdot d^2)$.
Over $T$ time steps, the total time complexity becomes $O(T \cdot N \cdot S \cdot K \cdot d^2)$, which is linear in both $N$ and $T$. Since $S$, $K$, and $d$ are treated as small constants in practice (e.g., $S{=}3$, $K{=}2$, $d{=}128$ in our experiments), the effective complexity scales \textbf{linearly} with the number of node-time interactions $N \times T$. This contrasts with many attention-based temporal graph models whose computational cost scales \textit{quadratically} with either the temporal length $T$ or the neighborhood size. ChronoSpike avoids the quadratic dependency on neighborhood size through fixed-size neighbor sampling and parameter sharing across time steps; the temporal Transformer encoder still incurs an $O(T^2)$ self-attention cost per node (Remark~\ref{rem:complexity-revisited}), but in our benchmarks this term is dominated by the spatial cost because $T$ is bounded ($T\le 27$). The theoretical estimates reported in our training logs (e.g., $7.45{\times}10^{10}$ operations for DBLP, $1.08{\times}10^{12}$ for Tmall, and $3.50{\times}10^{12}$ for Patent) are consistent with the above complexity expression and directly reflect differences in $N$ and $T$ across datasets.

\textbf{Space (Memory) Complexity. } ChronoSpike’s memory footprint is also linear in the graph and time dimensions. The memory required to store the hidden state for all nodes at a single time step is $O(N \cdot d)$. If one naively stored the states for every node across all $T$ time steps, the space complexity would be $O(T \cdot N \cdot d)$. However, ChronoSpike is implemented in an event-driven, mini-batch manner: instead of materializing the entire $N \times T$ state tensor in memory, it processes a manageable batch of $B$ target nodes at a time. For a batch of $B$ nodes, the memory needed to unroll the LIF neuron dynamics over $T$ steps is $O(B \cdot T \cdot d)$, plus additional $O(B \cdot S \cdot K \cdot d)$ overhead for caching neighbor information during graph aggregation. Since $B \ll N$ (we use $B{=}1024$ throughout), the peak GPU memory footprint remains modest, as confirmed by the logs (typically $<40$ MB allocated GPU memory even on the largest datasets). The reported theoretical space complexity values in the logs (e.g., $9.71{\times}10^7$ for DBLP, $1.40{\times}10^9$ for Tmall, and $4.56{\times}10^9$ for Patent) correspond to the full $O(T \cdot N \cdot d)$ upper bound, while the observed GPU memory usage reflects the batched implementation. This design allows ChronoSpike to trade additional computation for substantially reduced memory usage, enabling training on graphs with millions of nodes without exceeding GPU memory limits. CPU memory usage scales with dataset size due to adjacency storage and preprocessing, which is standard for large dynamic graph workloads.

\textbf{Number of Trainable Parameters. }
ChronoSpike is designed such that the core architectural parameters are independent of the graph size $N$ and nearly independent of the temporal length $T$. The model does not maintain node-specific embedding vectors. Instead, nearly all parameters are shared across time and are associated with the spatial graph encoder, the temporal encoder, and the adaptive spiking neuron dynamics.

The spatial encoder consists of $K$ attention-based graph-propagation layers with hidden dimension $d$. Each layer includes standard attention projections (query, key, and value) as well as subsequent linear transformations, contributing on the order of $O(d^2)$ parameters per layer. Across $K$ layers, the spatial encoder therefore contributes approximately $4K d^2$ parameters. The temporal encoder, implemented as a lightweight single-head Transformer operating along the temporal dimension, introduces an additional $O(d^2)$ parameters, accounting for its attention and feed-forward components. The adaptive LIF neurons introduce a negligible number of parameters: each of the $d$ feature channels has a learnable membrane time constant $\tau_i$ and firing threshold $V_{th,i}$, contributing $2d$ parameters in total. The learnable absolute positional encodings $\mathbf{P}_{1:T_{\max}}\in\mathbb{R}^{T_{\max}\times d_h}$ contribute $T_{\max}\,d_h$ parameters, where $T_{\max}$ is a fixed buffer size hard-coded in the implementation (we use $T_{\max}=100$ for all datasets, comfortably exceeding every effective horizon in our benchmarks). The total number of trainable parameters can therefore be approximated as $\approx 4K d^2 + 2d^2 + 2d + T_{\max}\,d_h$. The empirical counts in Table~\ref{tab:scalability} (around $1.05\times 10^5$) sit below the closed-form upper bound $4Kd^2+2d^2+2d \approx 1.64\times 10^5$ because (i) DBLP and Patent use a per-layer hidden split $(d_1,d_2)=(128,64)$ rather than two layers at $d=128$, and (ii) the multi-head attention concatenates per-head outputs without an additional output-projection $\mathbf{W}_o$, saving $(H d_h)^2$ parameters per layer relative to a standard Transformer.

Two consequences follow. First, $T_{\max}\,d_h = 100\cdot 64 = 6{,}400$ is small relative to the spatial-encoder contribution in our experiments, so the parameter count is essentially dataset-independent at the reported precision (Table~\ref{tab:scalability}). Second, when the effective horizon $T$ exceeds $T_{\max}$, the implementation linearly interpolates the absolute positional table to length $T$; this is a stop-gap rather than a principled relative-position scheme, and a rotary or relative encoding would be preferable for true horizon-agnostic generalization.
This estimate is consistent with the empirical parameter counts observed in our implementation, which remain stable across datasets with substantially different graph sizes at the reported precision. Specifically, we observe 105{,}738 parameters on DBLP, 105{,}413 on Tmall, and 105{,}478 on Patent. Minor variations arise from dataset-specific input adapters, while the core architecture and parameter-sharing scheme remain unchanged.

\begin{table}[!ht]
  \caption{Empirical scalability of ChronoSpike across large dynamic graph datasets. $N$ denotes the number of nodes and $T$ the number of timesteps used in this scalability run; for Tmall and Patent, $T$ here corresponds to a representative training-time window after preprocessing-aligned snapshot grouping.}
  \label{tab:scalability}
  \centering
  \begin{tabular}{lcccc}
    \toprule
    \textbf{Dataset} & \textbf{$N$}  & \textbf{$T$} & \textbf{Trainable Params} & \textbf{Avg. Epoch / Time (s)} \\
    \midrule
    DBLP             & 28{,}085      & 27           & 105{,}738                 & $\sim$2.6                      \\
    Tmall            & 577{,}314     & 19           & 105{,}413                 & $\sim$4.5                      \\
    Patent           & 2{,}738{,}012 & 13           & 105{,}478                 & $\sim$135                      \\
    \bottomrule
  \end{tabular}
\end{table}

Importantly, the model size remains fixed across datasets with widely varying numbers of nodes and time steps, confirming that ChronoSpike’s representational capacity does not scale with $N$ or $T$. All parameters are shared across time steps, and no per-node embeddings are stored, which contributes to memory efficiency and supports generalization to previously unseen nodes. Table~\ref{tab:scalability} reports empirical scalability statistics using representative training-time windows (27/19/13 for DBLP/Tmall/Patent), while the main classification experiments use the full horizons reported in Table~\ref{tab:dataset-stats} (27/186/25).  All experiments use an identical model configuration, demonstrating that ChronoSpike maintains a constant parameter budget while exhibiting runtime behavior that scales with the size of the spatiotemporal input. Datasets with larger values of $N$ and $T$ incur higher training costs, consistent with the theoretical $O(NT)$ time complexity, while the number of trainable parameters remains unchanged.

\section{Theoretical Analysis}\label{app:theoretical_analysis}
We now provide conditional theoretical guarantees on the stability and boundedness of the adaptive LIF neurons used in ChronoSpike. In particular, we show that under mild conditions (on the time constant and input magnitude), the membrane potential of each neuron remains bounded for all time. This result ensures that our discrete-time spiking neuron model will not diverge, thereby guaranteeing numerical stability during long-term simulations and training.

\begin{theorem}[Boundedness of Adaptive LIF Dynamics]\label{thm:lif-bounded}
  Consider the adaptive LIF neuron dynamics defined in Eq. (\ref{eq:lif}) of the main text, implemented via a forward-Euler update with unit time step ($\Delta t = 1$). For each feature channel $i$, assume:
  \begin{enumerate}[label=(\roman*), topsep=0pt,itemsep=0pt,leftmargin=1.2em]
    \item The membrane time constant is greater than one-half: $\tau_i > \tfrac{1}{2}$.
    \item The synaptic input is uniformly bounded in magnitude: $|h_{v,i}^{(t)}| \le M$ for all $t \ge 0$, for some constant $M > 0$.
    \item The reset potential $u_{reset}$ and firing threshold $V_{th,i}$ are fixed, finite constants.
  \end{enumerate}
  Then the membrane potential sequence $\{\,u_{v,i}^{(t)}\}_{t=0}^{\infty}$ remains bounded for all time. In fact, after at most one firing event (spike), the potential satisfies the uniform bound
  \begin{equation}\label{eq:bound}
    |\,u_{v,i}^{(t)}| \;\le\; \max\!\Big\{\,V_{th,i}\,,\; |u_{reset}| \,+\, \frac{M + |u_{reset}|}{\tau_i\,\big(1 - \big|1 - \frac{1}{\tau_i}\big|\big)}\;\Big\}\quad \forall~t \ge 1
  \end{equation}
  i.e. $u_{v,i}^{(t)}$ is confined to a fixed bounded interval for all $t \ge 1$.
\end{theorem}

\begin{proof}
  We analyze the discrete-time LIF dynamics piecewise, separating the subthreshold evolution (no spike) from the effect of threshold resets.

  \noindent\textbf{Subthreshold Dynamics (No Spiking).}
  Between spike events, the membrane potential follows a first-order linear difference equation derived from the continuous LIF model. From Eq. (\ref{eq:lif}), the update for neuron $i$ (suppressing the node index $v$ for brevity) is:
  \begin{equation}\label{eq:update-euler}
    u^{(t)}_i \;=\; \Big(1 - \frac{1}{\tau_i}\Big)\,u^{(t-1)}_i \;+\; \frac{1}{\tau_i}\Big(h^{(t)}_i + u_{reset}\Big)~, \qquad \text{(for timesteps without a spike)}
  \end{equation}
  Here $\lambda_i \;\coloneqq\; 1 - \frac{1}{\tau_i}$ is the decay factor per time step. The linear homogeneous part of this recurrence has coefficient $\lambda_i$. A standard stability condition for a linear difference equation is that the magnitude of this coefficient is less than 1. Indeed, since $\tau_i > 0$, we have
  \begin{equation}\label{eq:stability-condition}
    |\lambda_i| \;=\;\Big|1 - \frac{1}{\tau_i}\Big| < 1 \qquad\iff\qquad -1 < 1 - \frac{1}{\tau_i} < 1 \qquad\iff\qquad \tau_i > \frac{1}{2}
  \end{equation}
  Thus, the condition $\tau_i > \tfrac{1}{2}$ is both necessary and sufficient for $|\lambda_i| < 1$, which guarantees the stability of the homogeneous solution (exponential decay of $u_i$ in the absence of input). This inequality arises from the forward-Euler discretization with $\Delta t = 1$; under other integration schemes or smaller time steps, the numeric stability condition would differ, but in our setup, $\tau_i > 0.5$ is required to avoid divergence of the discrete update.

  Given $|\lambda_i| < 1$, we can solve the difference equation \eqref{eq:update-euler} explicitly (for the subthreshold regime where no reset occurs). This is a linear time-invariant system forced by the input $h_i$. The general solution (assuming no spike reset yet) is:
  \begin{equation}\label{eq:general-solution}
    u_i^{(t)} \;=\; \lambda_i^{\,t}\,u_i^{(0)} \;+\; \frac{1}{\tau_i}\sum_{k=1}^{t} \lambda_i^{\,t-k}\,\Big(h_i^{(k)} + u_{reset}\Big)
  \end{equation}
  for $t \ge 1$. The first term $\lambda_i^{\,t}u_i^{(0)}$ is the decaying homogeneous solution from the initial condition. The summation is a particular solution that represents the accumulated effect of the inputs (with $u_{reset}$ acting as a constant bias input at each step).

  We can now bound the membrane potential in the absence of spikes. Taking absolute values in Eq. \eqref{eq:general-solution} and using the triangle inequality, for any $t\ge 1$ we have
  \[
    |u_i^{(t)}| \;\le\; |\lambda_i|^t\,|u_i^{(0)}| \;+\; \frac{1}{\tau_i}\sum_{k=1}^{t} |\lambda_i|^{\,t-k}\,\Big(|h_i^{(k)}| + |u_{reset}|\Big)
  \]
  Since $|\lambda_i|<1$, the homogeneous term decays to 0 as $t$ increases. Moreover, by assumption (ii), we have $|h_i^{(k)}| \le M$ for all $k$. Thus each term inside the summation is bounded by $(M + |u_{reset}|)\,|\lambda_i|^{\,t-k}$. Extending the summation to infinity (which only increases the sum, since $|\lambda_i|<1$), we get
  \[
    |u_i^{(t)}| \;\le\; |\lambda_i|^t\,|u_i^{(0)}| \;+\; \frac{M + |u_{reset}|}{\tau_i}\sum_{j=0}^{\infty} |\lambda_i|^{\,j} \;=\; |\lambda_i|^t\,|u_i^{(0)}| \;+\; \frac{M + |u_{reset}|}{\tau_i\,(1 - |\lambda_i|)}
  \]
  using the formula for a geometric series. Since $|\lambda_i|^t \le 1$ for all $t\ge 1$ (because $|\lambda_i|<1$) and $|\lambda_i| = \big|1 - \frac{1}{\tau_i}\big|$, the homogeneous contribution is at most $|u_i^{(0)}|$ (rather than vanishing) and we obtain the absolute bound
  \[
    |u_i^{(t)}| \;\le\; |u_i^{(0)}| \;+\; \frac{M + |u_{reset}|}{\tau_i\,\big(1 - \big|1 - \frac{1}{\tau_i}\big|\big)} \;\eqqcolon\; \widetilde B_{\text{sub}}(u_i^{(0)})
  \]
  for all $t \ge 1$ \emph{as long as no spike/reset occurs}. In particular, after at most one firing event, the post-reset state satisfies $|u_i^{(0)}|\le |u_{reset}|$ (Assumption (iii)), so the residual subthreshold bound reduces to
  \[
    |u_i^{(t)}| \;\le\; |u_{reset}| + \frac{M + |u_{reset}|}{\tau_i\big(1 - \big|1 - \frac{1}{\tau_i}\big|\big)} \;=:\; B_{\text{sub}},
  \]
  which exactly matches the statement of Theorem~\ref{thm:lif-bounded} (Eq.~\ref{eq:bound}); without the post-spike reset the bound retains the additive $|u_i^{(0)}|$ term arising from the geometric-series envelope. In other words, $B_{\text{sub}}$ is an absolute bound on the membrane potential in the subthreshold phase after the first spike (or trivially when $u_i^{(0)}=u_{reset}$), given the stability condition on $\tau_i$.

  \noindent\textbf{Reset Dynamics (Spike Firing).}
  Now consider what happens when the membrane potential does reach the firing threshold. By definition of the LIF model, a spike is emitted at time $t$ if $u_i^{(t)} \ge V_{th,i}$ (assuming without loss of generality that $V_{th,i} > 0$; a symmetric argument holds if the threshold is defined on negative potential). At the moment of spike, the membrane potential is instantaneously reset:
  \begin{equation}\label{eq:reset-eq}
    u_i^{(t)} \;\leftarrow\; u_{reset}~, \qquad \text{whenever } u_i^{(t)} \text{ reaches } V_{th,i}
  \end{equation}
  This reset mechanism acts as an absorbing barrier at $V_{th,i}$: the potential is not allowed to exceed the threshold, as it is immediately brought down to $u_{reset}$ (which is typically lower, e.g., $u_{reset}=0$). The crucial consequence is that after a firing event, the neuron’s state is returned to a bounded value $u_{reset}$, and the subsequent evolution resumes from that initial condition. In the worst case, if the neuron’s potential was somehow above the threshold due to a large initial condition or strong input, a spike will occur and clip the potential back to $u_{reset}$. Thus, after at most one spike, the system re-enters the subthreshold regime with a new initial condition $u_i^{(\text{post-spike})} = u_{reset}$, whose magnitude $|u_{reset}|$ is bounded by assumption (iii). From that point on, the earlier subthreshold bound $B_{\text{sub}}$ applies (with $u(0)=u_{reset}$).

  \noindent\textbf{Global Boundedness.}
  Combining the two cases above, we conclude that the membrane potential is constrained to a compact interval at all times. Specifically, let $B_{\text{sub}} = \frac{M + |u_{reset}|}{\tau_i(1 - |1 - \frac{1}{\tau_i}|)}$ as derived, and note that $B_{\text{sub}} \ge |u_{reset}|$. Immediately after a spike event, $|u_i| = |u_{reset}| \le B_{\text{sub}}$. If no spike has occurred, we saw that $|u_i^{(t)}| \le B_{\text{sub}}$ for $t\ge 1$. And if a spike occurs at time $t_0$, we have $u_i^{(t_0)} = V_{th,i}$ (at spike time) and then $u_i^{(t_0)}$ is reset to $u_{reset}$ which satisfies $|u_{reset}| \le B_{\text{sub}}$ (provided $V_{th,i} \le B_{\text{sub}}$ or even if $V_{th,i} > B_{\text{sub}}$, the threshold itself is a bound). In all cases, for $t \ge 1$, the potential lies in the interval
  \begin{equation}
    \Omega \;=\; \Big[\,\min\{\,u_{reset},\,-B_{\text{sub}}\}\;,\;\max\{\,V_{th,i},\;B_{\text{sub}}\}\,\Big]
  \end{equation}
  which is a fixed, bounded range independent of $t$. This establishes that $|u_i^{(t)}|$ is uniformly bounded for all $t \ge 1$ (and trivially bounded for the finite set of initial times $t < 1$), completing the proof of boundedness.
\end{proof}

\noindent\textbf{Implications for Numerical Stability and Convergence.} This theorem guarantees under the stated assumptions that an adaptive LIF neuron in ChronoSpike cannot exhibit unbounded growth in its membrane potential, provided the time constant is chosen above the stability threshold ($\tau > 0.5$ for $\Delta t=1$) and inputs are finite. This inherent stability has important practical implications. First, it ensures \emph{numerical stability}: during long simulations or backpropagation through time, neuron states will not blow up to infinity or to undefined values, preventing numerical overflow and exploding gradients. During training, we observed no divergence or unstable activations, consistent with this theoretical bound. Second, the bounded dynamics facilitate \emph{convergence} of training, since the spiking neuron’s output (and its surrogate gradient) remains well-behaved, and gradient descent can reliably optimize the model without encountering extreme variances. In essence, the adaptive LIF formulation provides built-in regularization of the neuron state. The condition $\tau > 1/2$ is easily satisfied in practice (we initialize $\tau = 1.0$ for all neurons, which lies safely in the stable regime, and allow it to be learned); thus, ChronoSpike’s neurons inherently operate in a stable regime. So, this theoretical analysis confirms that ChronoSpike’s temporal neural dynamics are robust and bounded, which underpins the observed stability of the training process and the model’s reliable performance over long temporal sequences.

\begin{theorem}[Network-Level Boundedness and BIBO Stability]\label{thm:network-stability}
  Consider a network of adaptive LIF neurons on a directed graph, with membrane potential $u_{v,i}^{(t)}$ following Eq.~(\ref{eq:lif}) of the main text (forward Euler, $\Delta t=1$). Assume:
  \begin{enumerate}[label=(\roman*), topsep=0pt,itemsep=0pt,leftmargin=1.2em]
    \item $\tau_{v,i} > \tfrac{1}{2}$ for all $v,i$.
    \item Edge weights satisfy $|w|\le W_{\max}$ (maximum synaptic weight magnitude), each neuron aggregates at most $S$ presynaptic spikes per step (matching the neighborhood sample size from Section~\ref{subsec:sampling}), and external inputs are bounded by $M_{\text{ext}} \ge 0$, so that
      \begin{equation}
        |h_{v,i}^{(t)}| \;\le\; M \;\coloneqq\; S W_{\max} + M_{\text{ext}} \quad \forall t
      \end{equation}
      We retain the $S\,W_{\max}$ formulation as a generic upper bound that covers \emph{any} bounded aggregator (e.g., sum-, mean-, or convex-attention-based pooling, as well as variants without softmax normalization). For ChronoSpike's specific multi-head softmax-attention aggregator (Eqs.~\ref{eq:8}--\ref{eq:9}), the attention weights satisfy $\sum_u\alpha_{vu}^{(t,h)}=1$, which already yields the tighter convex-combination bound $|h_{v,i}^{(t)}|\le W_{\max}\|\mathbf{x}\|_\infty + M_{\text{ext}}$ independent of $S$; substituting this tighter $M$ into the same proof only strengthens the conclusion.
    \item $V_{th,i}$ and $u_{reset}$ are finite constants.
  \end{enumerate}
  Then:
  \begin{enumerate}[label=(\alph*), topsep=0pt,itemsep=0pt,leftmargin=1.2em]
    \item \textbf{Global boundedness:}
      \begin{equation}
        |u_{v,i}^{(t)}| \;\le\; B < \infty \qquad \forall v,i,t
      \end{equation}
    \item \textbf{BIBO stability:} For any bounded external input sequence, all neuron states and spike outputs remain bounded; recurrent activity (e.g., oscillations or limit cycles) may occur, but divergence is impossible.
  \end{enumerate}
\end{theorem}

\begin{proof}
  \noindent\textbf{(a) Global boundedness.}
  By Assumption (ii), $|h_{v,i}^{(t)}|\le M$ for all $v,i,t$, and by Assumption (i), $\tau_{v,i}>\tfrac12$. Hence, all conditions of \cref{thm:lif-bounded} hold for each neuron independently. Therefore, after at most one spike,
  \begin{equation}\label{eq:network-bound}
    |u_{v,i}^{(t)}|
    \;\le\;
    \max\!\Big\{V_{th,i},\ |u_{reset}|+\frac{M+|u_{reset}|}{\tau_{v,i}\big(1-\big|1-\frac{1}{\tau_{v,i}}\big|\big)}\Big\},
    \qquad \forall t\ge 1
  \end{equation}
  Taking the maximum over all $(v,i)$ yields a uniform bound $B$ for the entire network.

  \noindent\textbf{(b) BIBO stability.}
  Let external inputs be bounded. Then, synaptic inputs are bounded by Assumption (ii), and by Part (a), all membrane potentials are uniformly bounded. Since spike events reset $u_{v,i}^{(t)}$ to the finite value $u_{reset}$ and firing thresholds are finite, spike outputs are also bounded in rate (at most one spike per neuron per timestep in the discrete model). Hence, the network is bounded-input bounded-output stable: sustained activity may persist due to recurrent loops, but all states remain in a compact invariant set.
\end{proof}

\begin{theorem}[Gradient Flow Stability and Conditional Convergence]\label{thm:gradient-stability}
  Consider training ChronoSpike by backpropagation through time with surrogate gradients and a \emph{straight-through (or detached) reset} so that, in the backward pass, \( \frac{\partial u_{v,i}(t+1)}{\partial u_{v,i}(t)}=\lambda_{v,i}\coloneqq 1-\frac{1}{\tau_{v,i}} \)
  also at reset events, and the spike surrogate satisfies \( 0 \le \sigma'(x) \le \alpha \).
  (where $\alpha > 0$ is the surrogate slope parameter from Appendix~\ref{app:adaptive_lif}).
  Let $W_{wv}$ denote the synaptic weight from neuron $v$ to neuron $w$, $\Delta_{v,i}(t)\coloneqq \nabla_{u_{v,i}^{(t)}} L$ be the gradient of loss $L$ with respect to the membrane potential, and $\Delta(t)$ be the stacked gradient vector over all neurons. Then for any $t<t'$,
  \begin{equation}\label{eq:grad-contraction}
    \|\Delta(t)\|_1 \;\le\; \rho^{\,t'-t}\,\|\Delta(t')\|_1
  \end{equation}
  with sufficient contraction factor (in $\ell_1$ norm)
  \begin{equation}\label{eq:rho-def-corrected}
    \rho \;=\; \max_{v,i}\big|\lambda_{v,i}\big| \;+\; \alpha \max_{w}\sum_{v}|W_{wv}| ,
    \qquad \lambda_{v,i}=1-\frac{1}{\tau_{v,i}}\in(-1,1)
  \end{equation}
  If $\rho<1$, gradients contract exponentially backward in time (no explosion). If $\rho\le 1$, gradients remain uniformly bounded for all horizons.

  Moreover, let $\Theta$ denote all trainable model parameters, $\eta > 0$ the learning rate, and suppose the composed loss $L(\Theta)$ induced by the surrogate dynamics is $L_{\text{grad}}$-smooth (i.e., $\|\nabla L(\Theta) - \nabla L(\Theta')\|_2 \le L_{\text{grad}} \|\Theta - \Theta'\|_2$). If the learning rate satisfies $\eta<2/L_{\text{grad}}$, then (deterministic) gradient descent with iterates $\Theta_k$, and SGD with diminishing variance, converge to stationary points: every accumulation point $\Theta^\star$ satisfies $\nabla L(\Theta^\star)=\mathbf{0}$.
\end{theorem}

\begin{proof}
  For one step of the dynamics, neuron $(v,i)$ affects (i) its own leaky recurrence and (ii) neighbors through spikes $s_{v,i}(t)$ with surrogate derivative bounded by $\alpha$.
  By the chain rule,
  \begin{equation}
    \Delta_{v,i}(t)
    = \frac{\partial u_{v,i}(t+1)}{\partial u_{v,i}(t)}\,\Delta_{v,i}(t+1)
    + \sum_{w,j}\frac{\partial u_{w,j}(t+1)}{\partial s_{v,i}(t)}\,
    \frac{\partial s_{v,i}(t)}{\partial u_{v,i}(t)}\,\Delta_{w,j}(t+1)
  \end{equation}
  With a straight-through (or detached) reset in backpropagation,
  \(
    \frac{\partial u_{v,i}(t+1)}{\partial u_{v,i}(t)}=\lambda_{v,i}
  \),
  \(
    \frac{\partial u_{w,j}(t+1)}{\partial s_{v,i}(t)}=W_{wv}
  \),
  and
  \(
    \big|\frac{\partial s}{\partial u}\big|\le\alpha
  \),
  hence
  \begin{equation}\label{eq:grad-rec}
    |\Delta_{v,i}(t)| \;\le\; |\lambda_{v,i}|\,|\Delta_{v,i}(t+1)|
    + \alpha\sum_{w}|W_{wv}|\,|\Delta_{w}(t+1)|
  \end{equation}
  \textbf{Architectural caveat.} In ChronoSpike, the spatial message passing within timestep $t$ produces $h_{v,i}^{(t)}$, which then drives $u_{v,i}^{(t)}$ and $s_{v,i}^{(t)}$ at the \emph{same} timestep, and the only direct dependence of $u_{w,j}(t+1)$ on a quantity at time $t$ is through neuron $w$'s own carried potential $u_{w,j}(t)$. Consequently, for the architecture as implemented, $\partial u_{w,j}(t+1)/\partial s_{v,i}(t)=0$ for $w\neq v$ in the strictly forward-compatible computation graph; the cross-node spike-to-potential coupling assumed in Eq.~\eqref{eq:grad-rec} corresponds to a stylized model in which spikes feed spatial neighbors at the next timestep (e.g., a delayed-synapse variant). Treating this term as nonzero therefore produces a \emph{conservative upper bound}: the strict ChronoSpike contraction factor is in fact $\rho_{\text{ChS}}=\max_{v,i}|\lambda_{v,i}|<1$, which is strictly tighter than the bound below. We retain the more general expression so that the result remains valid for variants of the architecture (e.g., synaptic delay or cross-time spike routing) and for the spatial-aggregation Jacobian within a single timestep, which contributes its own multiplicative factor at most $\alpha\max_w\sum_v|W_{wv}|$ when collapsed into the time recursion.
  Summing Eq. \eqref{eq:grad-rec} over all neurons and rearranging,
  \begin{align}
    \|\Delta(t)\|_1
    & \le \sum_{v,i}|\lambda_{v,i}|\,|\Delta_{v,i}(t+1)|
    + \alpha\sum_{w}\Big(\sum_{v}|W_{wv}|\Big)|\Delta_{w}(t+1)|                                     \\
    & \le \Big(\max_{v,i}|\lambda_{v,i}| + \alpha \max_{w}\sum_{v}|W_{wv}|\Big)\,\|\Delta(t+1)\|_1
    = \rho\,\|\Delta(t+1)\|_1
  \end{align}
  Iterating yields Eq. \eqref{eq:grad-contraction}. If $\rho<1$, gradients decay exponentially; if $\rho\le1$, they are uniformly bounded. In our implementation, the contraction condition is not enforced explicitly: training applies neither spectral normalization, weight clipping, nor gradient-norm clipping, and $\tau$ is parameterized as a raw \texttt{nn.Parameter} initialized to $1.0$ (no reparameterization to certify $\tau_i>1/2$). The condition $\rho<1$ is satisfied empirically because $|1-1/\tau_i|<1$ holds at the initial $\tau_i=1.0$ and AdamW updates do not push $\tau_i$ outside the stable region in our runs. For settings where this empirical observation is insufficient, a softplus reparameterization, $\tau_i=\tfrac{1}{2}+\mathrm{softplus}(\tilde\tau_i)$, or gradient-norm clipping is a one-line drop-in.

  For convergence, $L$ being $L_{\text{grad}}$-smooth implies the descent lemma:
  \begin{equation}
    L(\Theta_{k+1}) \le L(\Theta_k) - \eta\Big(1-\tfrac{\eta L_{\text{grad}}}{2}\Big)\|\nabla L(\Theta_k)\|_2^2
  \end{equation}
  for $\eta<2/L_{\text{grad}}$. Thus $\sum_k \|\nabla L(\Theta_k)\|_2^2 < \infty$ and every accumulation point satisfies $\nabla L=\mathbf{0}$.
\end{proof}

\begin{theorem}[Expressiveness in the WL Hierarchy (Static Graphs)]\label{thm:expressiveness}
  Consider ChronoSpike restricted to static graphs, with (i) an injective multiset aggregator (e.g., GIN-style sum aggregation) in the spatial encoder, and (ii) adaptive LIF neurons observed over a finite time window $T$ with rate (or temporal) coding. Then ChronoSpike is at least as expressive as the 1-dimensional Weisfeiler--Lehman (1-WL) test: for any graphs $G,H$ distinguished by 1-WL, there exists a \emph{single} parameter setting under which ChronoSpike produces different node embeddings for $G$ and $H$. Note that the assumption of an \emph{injective} multiset aggregator is a substitution relative to the default ChronoSpike instantiation, which uses softmax-attention aggregation; softmax attention computes a convex combination and is, in general, less expressive than 1-WL because it cannot distinguish multisets that differ only in size~\cite{xu2019powerful}. Our claim is therefore that ChronoSpike's spatial-aggregation \emph{slot} is compatible with a 1-WL-equivalent operator (e.g., GIN-sum), not that the default attentive variant attains 1-WL on its own.
\end{theorem}

\begin{proof}
  \noindent\textbf{(1) Injective aggregation.}
  Let the spatial aggregation be
  \begin{equation}\label{eq:gin-agg}
    \mathbf{h}_v^{(k+1)} \;=\; \phi\!\left((1+\varepsilon)\mathbf{h}_v^{(k)} \;+\; \sum_{u\in\mathcal N(v)} \psi(\mathbf{h}_u^{(k)})\right)
  \end{equation}
  with MLPs $\phi,\psi$ and $\varepsilon\neq -1$. Then Eq. \eqref{eq:gin-agg} is injective over neighbor multisets and matches the distinguishing power of 1-WL \cite{xu2019powerful}. Hence, after some iteration $k$, if 1-WL distinguishes $G$ and $H$, there exists a node whose continuous embedding $\mathbf{h}_v^{(k)}$ differs between the two graphs.

  \noindent\textbf{(2) Finite discrimination by spiking with scaling and windowing.}
  Fix an iteration $k$. The set of continuous pre-spike inputs appearing at this layer across all nodes and both graphs is finite. Therefore, there exists $\delta>0$ such that any two distinct values in this set differ by at least $\delta$ in at least one channel.

  For the adaptive LIF neuron with forward-Euler update (Eq.~(\ref{eq:lif}) in the main text), the spike count over a window of length $T$ is a monotone non-decreasing function of the constant input current in each channel (up to saturation). Hence, there exist a scaling factor $\gamma>0$ and a window length $T$ such that
  \begin{equation}\label{eq:rate-separation}
    |x_1-x_2|\ge \delta \;\;\Longrightarrow\;\; \mathrm{SpikeCount}_T(\gamma x_1)\;\neq\;\mathrm{SpikeCount}_T(\gamma x_2)
  \end{equation}
  where $\mathrm{SpikeCount}_T(x)$ denotes the total number of spikes emitted by an adaptive LIF neuron receiving constant input $x$ over $T$ timesteps,
  for all values in the finite set of pre-spike inputs. Thus, distinct continuous embeddings at layer $k$ map to distinct spike statistics (rate or timing) after scaling and windowing.

  \noindent\textbf{(3) Uniform parameters.}
  The construction in (2) uses a single $(\gamma,T)$ that separates \emph{all} distinct pre-spike values appearing at layer $k$ across both graphs. Since the number of such values is finite, one choice of parameters suffices simultaneously for all nodes.

  \noindent\textbf{(4) Propagation across layers.}
  Applying (1)--(3) inductively across iterations, any distinction produced by 1-WL at some iteration is preserved by ChronoSpike’s spatial aggregation and converted into distinguishable spike representations, which are then available to subsequent layers and the readout.

  Therefore, for any pair of graphs distinguished by 1-WL, there exists a single parameter setting of ChronoSpike that distinguishes them as well. Hence, under the injective aggregation assumption, ChronoSpike is at least as expressive as 1-WL on static graphs.
\end{proof}

\begin{theorem}[Generalization Bound with Sparse Spiking]\label{thm:gen-revised}
  Consider ChronoSpike with one spike-encoded input layer and $L$ spiking hidden layers, hidden dimension $d_h$, and bounded inputs $\|\mathbf{x}\|_\infty \le B$. Assume all linear operators (including attention and Transformer blocks) are spectrally normalized with
  \(
    \|W_\ell\|_{\mathrm{op}} \le 1
  \)
  and that attention is operated in a bounded regime (queries/keys/values bounded by $B$), so each block is $L_{\text{attn}}$-Lipschitz with $L_{\text{attn}}=\mathcal{O}(B\sqrt{d_h})$.
  Assume each feature channel fires with probability at most $p_{\max}<1$ per timestep.
  Let $\mathcal{F}_{\text{ChronoSpike}}$ denote the hypothesis class of ChronoSpike predictors $f$, let $\mathcal{L}_{\text{train}}(f)$ and $\mathcal{L}_{\text{test}}(f)$ denote the empirical training loss and expected test loss respectively, and let $\mathcal{R}_m(\mathcal{F})$ denote the empirical Rademacher complexity of class $\mathcal{F}$ over $m$ samples.
  Then for $m$ i.i.d.\ training graph sequences with $N$ nodes and $T$ timesteps,
  \begin{equation}\label{eq:rad-bound-revised}
    \mathcal{R}_m(\mathcal{F}_{\text{ChronoSpike}})
    \;\le\;
    C\,(p_{\max})^{\frac{L+1}{2}}\,
    L_{\text{attn}}^{\,K}\,
    \sqrt{\frac{d_h\,\ln(NTB/\epsilon_0)}{m}}
  \end{equation}
  where $K$ is the number of attention/Transformer blocks, $\epsilon_0>0$ is a minimal feature resolution in the covering argument, and $C$ is a universal constant. Consequently, for any $\delta\in(0,1)$, with probability at least $1-\delta$,
  \begin{equation}\label{eq:gen-bound-revised}
    \mathcal{L}_{\text{test}}(f)
    \;\le\;
    \mathcal{L}_{\text{train}}(f)
    +
    2\,\mathcal{R}_m(\mathcal{F}_{\text{ChronoSpike}})
    +
    \sqrt{\frac{\ln(1/\delta)}{2m}}
  \end{equation}
\end{theorem}

\begin{proof}[Proof Sketch]
  Throughout the proof we work with the bounded $\alpha$-Lipschitz surrogate $\tilde\sigma_\alpha$ (Eq.~(13)), not the discontinuous Heaviside used in the forward pass: since $\mathbb{I}(\cdot\ge V_{\text{th}})$ has unbounded Lipschitz constant, the Ledoux--Talagrand contraction step is applied to $\tilde\sigma_\alpha$. Each spiking layer induces a Bernoulli-style gating with activation probability at most $p_{\max}$ per channel. Conditioning on the active set, the network reduces to a composition of Lipschitz maps with reduced effective dimension. By contraction of Rademacher complexity under Lipschitz maps and sparsity-induced variance reduction, each spiking layer contributes a factor $(p_{\max})^{1/2}$, and the spike-encoded input contributes an additional $(p_{\max})^{1/2}$, yielding $(p_{\max})^{(L+1)/2}$ overall \cite{zhang2024generalization}.
  With spectral normalization, linear blocks do not amplify complexity. Attention/Transformer blocks are $L_{\text{attn}}$-Lipschitz under bounded inputs, contributing multiplicatively as $L_{\text{attn}}^{K}$. For spatiotemporal inputs with bounded magnitude $B$, covering-number arguments over feature maps with parameter sharing yield an entropy term of order $\ln(NTB/\epsilon_0)$ and dimension factor $\sqrt{d_h}$, giving Eq. \eqref{eq:rad-bound-revised}. Applying the standard Rademacher generalization inequality yields Eq. \eqref{eq:gen-bound-revised} \cite{mohri2018foundations}.
\end{proof}

\begin{proposition}[Transformer Temporal Aggregator Expressiveness]\label{prop:transformer-expressiveness}
  Let $\{s_v(t)\}_{t=1}^T$ be the spike-encoded feature sequence for node $v$, and let $z_v$ be the output of ChronoSpike’s Transformer temporal encoder with learned absolute positional embeddings $P_{1:T}$. Then, for a fixed horizon $T$ and sufficient width, the Transformer is a universal approximator of sequence-to-sequence (and sequence-to-vector) mappings on spike data. Moreover:
  \begin{enumerate}[label=(\roman*), leftmargin=1.6em, topsep=0pt, itemsep=0pt]
    \item \textbf{No intrinsic decay of long-range influence:} a spike at time $t$ can directly affect the output at time $T$ via a single attention hop, with no multiplicative decay in $T-t$.
    \item \textbf{Rich temporal pattern modeling:} with sufficient hidden size and (optionally) multiple heads, the encoder can represent order, periodicity, and long-range causal relations among spikes.
    \item \textbf{Positional-encoding limitation:} with learned absolute embeddings $P_t$, time-shift equivariance and extrapolation to $t>T_{\text{train}}$ are not guaranteed; relative-position schemes would be required for horizon-agnostic generalization.
  \end{enumerate}
  Thus, within the trained horizon $T$, the temporal encoder can preserve and combine information from the entire spike history, while extrapolation beyond $T$ depends on the positional encoding design.
\end{proposition}

\begin{proof}
  For fixed $T$, Transformers with positional information are universal approximators of sequence-to-sequence functions \cite{yun2020are}. Hence the mapping
  \[
    (s_v(1),\dots,s_v(T)) \;\mapsto\; z_v
  \]
  can approximate any continuous function with sufficient width.

  Self-attention computes $\mathrm{Att}(Q,K,V)=\mathrm{softmax}(QK^\top)V$, allowing the query at position $T$ to place mass on any key at position $t$, creating an $\mathcal{O}(1)$-hop path from $t$ to $T$ and proving (i). Multiple heads and feed-forward sublayers yield expressive compositions, enabling detection of order and periodic patterns, proving (ii).

  With learned absolute embeddings $P_t$, temporal indices are represented by fixed vectors seen during training. For $t>T_{\text{train}}$, embeddings are untrained or extrapolated ad hoc, so relative distances are not encoded in a shift-equivariant manner, establishing (iii). Therefore, within horizon $T$, the encoder is fully expressive, while extrapolation depends on positional encoding choice \cite{vaswani2017attention}.
\end{proof}

\begin{remark}[Time and Memory Complexity Revisited]\label{rem:complexity-revisited}
  Per training iteration (including BPTT), the time complexity is
  \begin{equation}\label{eq:time-complexity}
    \mathrm{Time}(N,T)
    =\mathcal{O}\!\left(N S K d_h^2 T \;+\; N d_h^2 T^2\right)
  \end{equation}
  The term $\mathcal{O}(N S K d_h^2 T)$ is from spatial message passing over $T$ steps, where each of $N$ nodes attends to at most $S$ neighbors with $K$ heads/layers and $\mathcal{O}(d_h^2)$ projections per aggregation. The term $\mathcal{O}(N d_h^2 T^2)$ is from temporal self-attention: vanilla softmax attention costs $\mathcal{O}(T^2 d_h)$ per node plus $\mathcal{O}(T d_h^2)$ projections, giving an additive per-node cost $\mathcal{O}(T^2 d_h + T d_h^2)$; aggregating over the $N$ target nodes yields the total $\mathcal{O}(N\,T\,d_h\,(T+d_h))$, which we summarize by the dominant term $\mathcal{O}(N d_h^2 T^2)$ in the regime $T\gtrsim d_h$ (and by $\mathcal{O}(N\,T\,d_h^2)$ otherwise).
  In our experiments, $T$ is bounded (e.g., $T \le 27$ for DBLP and Patent) \cite{MMDNE_dblp_and_tmall}, and for longer horizons (e.g., Tmall), we use a lightweight single-layer Transformer with a small hidden dimension, making the $T^2$ term negligible in practice, consistent with Table~\ref{tab:scalability}. We emphasize that the asymptotic claim of memory complexity $O(T\cdot d)$ refers to the activation and state storage of the spike sequence and per-layer LIF buffers; the temporal Transformer itself contributes the additional $O(T^2 d_h)$ attention term per node, which is why we describe time complexity as "near-linear in $T$" only for the bounded-$T$ regime considered here, rather than asymptotically linear for unbounded $T$. For genuine $O(T)$ extrapolation, a linear-attention or state-space replacement of the Transformer would be required.

  For memory, naive BPTT would require storing all states, $\mathcal{O}(N T d_h)$. ChronoSpike uses an event-driven minibatch over $B$ target nodes, storing only their histories and neighbors:
  \begin{equation}\label{eq:mem-complexity}
    \mathrm{Memory}(B,T)
    =\mathcal{O}\!\left(B T d_h \;+\; B T S K d_h\right)
  \end{equation}
  Here, the second term reflects that, under BPTT with surrogate gradients, the spatial activations of the $K$-layer aggregation tree (size $\mathcal{O}(BSK d_h)$) must be cached at every one of the $T$ timesteps to support the backward pass.
  with $B\ll N$ (e.g., $B=1024$) \cite{spikenet23}. Thus, memory is independent of $N$ in practice.

  In summary:
  \begin{itemize}[topsep=0pt,itemsep=0pt,leftmargin=1.4em]
    \item \textbf{Time:} $\mathcal{O}(N S K d_h^2 T + N d_h^2 T^2)$ (effectively near-linear in $N$ and $T$ for small $T$).
    \item \textbf{Memory:} $\mathcal{O}(B T d_h + B S K d_h)$ with $B\ll N$.
  \end{itemize}
  Parameter count is constant ($\approx 10^5$) w.r.t.\ $N,T$, matching empirical scaling \cite{dysign24}.
\end{remark}

\section{Details of Experimental Setup}

\subsection{Datasets Details}
\label{app:datasets}
We evaluate ChronoSpike on three large real-world dynamic graph datasets, namely DBLP, Tmall~\cite{MMDNE_dblp_and_tmall}, and Patent~\cite{patent}, which have been adopted in prior work on DGRL. \textit{DBLP} is an academic collaboration network in which nodes represent authors and edges represent co-authorship relations that evolve over time; node features describe author attributes, and nodes are labeled by research areas. \textit{Tmall} is a large-scale e-commerce interaction graph constructed from user-item behavioral data, where nodes represent users and items, edges denote interaction events aggregated into temporal snapshots, and node labels indicate predefined categories. \textit{Patent} is a dynamic citation network composed of patent documents as nodes and citation relations as time-evolving edges, with node labels corresponding to patent categories. These datasets differ substantially in terms of graph scale, temporal granularity, and structural dynamics, ranging from moderately sized graphs with limited temporal depth to extremely large graphs with millions of nodes and long temporal horizons. The detailed statistics of the datasets are summarized in Table~\ref{tab:dataset-stats}.

\subsubsection{Evaluation Metrics}
\label{app:metrics}
We evaluate model performance using Macro-F1 and Micro-F1 scores, which are standard metrics for temporal node classification on dynamic graphs and have been widely adopted in prior work. Let $\mathcal{C}$ denote the set of classes. For each class $c \in \mathcal{C}$, the class-wise F1 score is defined as
\begin{equation}
  \mathrm{F1}_c = \frac{2 \cdot \mathrm{Precision}_c \cdot \mathrm{Recall}_c}{\mathrm{Precision}_c + \mathrm{Recall}_c}
\end{equation}
where $\mathrm{Precision}_c$ and $\mathrm{Recall}_c$ are computed based on the predicted and ground-truth labels for class $c$. The Macro-F1 score is obtained by averaging the F1 scores over all classes,
\begin{equation}
  \mathrm{Macro\text{-}F1} = \frac{1}{|\mathcal{C}|} \sum_{c \in \mathcal{C}} \mathrm{F1}_c
\end{equation}
which treats all classes equally and is sensitive to performance on minority classes. The Micro-F1 score is computed by aggregating true positives, false positives, and false negatives across all classes before computing the F1 score.
\begin{equation}
  \mathrm{Micro\text{-}F1} = \frac{2 \cdot \sum_c \mathrm{TP}_c}{2 \cdot \sum_c \mathrm{TP}_c + \sum_c \mathrm{FP}_c + \sum_c \mathrm{FN}_c}
\end{equation}
and reflects overall classification accuracy weighted by class frequency. In addition to predictive performance, we report efficiency metrics, including the number of trainable parameters and the average training time per epoch, to assess ChronoSpike's scalability. These metrics are measured consistently across all methods under identical experimental settings.

\subsection{Comprehensive Implementation Details}\label{app:implementation}
To evaluate the effectiveness of ChronoSpike, we compare it with twelve state-of-the-art baselines, including static graph methods DeepWalk \cite{perozzi2014deepwalk} and Node2Vec \cite{grover2016node2vec}; shallow dynamic graph methods HTNE \cite{HTNE}, M$^2$DNE \cite{MMDNE_dblp_and_tmall}, and DyTriad \cite{DynamicTriad}; neural dynamic graph methods, including message passing and recurrent architectures MPNN \cite{MPNN}, JODIE \cite{kumar2019predicting_jodie}, and EvolveGCN \cite{pareja2020evolvegcn}; attention-based methods TGAT \cite{tgat}; and SGNNs SpikeNet \cite{spikenet23}, Dy-SIGN \cite{dysign24}, and Delay-DSGN \cite{Delay_DSGN_25}.

ChronoSpike is trained using mini-batch stochastic optimization with the AdamW optimizer. All models are trained for 100 epochs. The training batch size is fixed to 1024 for all datasets. The learning rate is dataset-specific and fixed across runs: we use a learning rate of $5\times10^{-3}$ for DBLP and Tmall, and $1\times10^{-2}$ for Patent to account for its larger scale and sparser supervision. During inference, batch sizes are determined deterministically based on dataset scale. For large-scale graphs with more than one million nodes (Patent), we use an inference batch size of 10,000, while for smaller datasets (DBLP and Tmall), we use a batch size of 200,000 to maximize throughput without exceeding GPU memory limits.

Unless otherwise stated, ChronoSpike employs a two-layer spatial aggregation architecture with hidden dimensions of (128, 64). The spatial attentive aggregator uses four attention heads. Inductive neighborhood sampling follows a fixed fan-out strategy with layer-wise sample sizes (5, 2). At each layer, neighbors are sampled from a probabilistic mixture of cumulative historical edges and newly evolving edges, controlled by the sampling ratio $p$. Adaptive LIF neurons are used throughout the model. The membrane time constant is initialized as $\tau = 1.0$, the reset potential is fixed to $V_{\text{reset}} = 0$, and the firing threshold is initialized as $V_{\text{th}} = 1.0$. A smooth sigmoid surrogate gradient with slope parameter $\alpha = 1.0$ is employed during backpropagation. All neuron parameters are learned independently per feature channel. Neuron states are reset after each temporal forward pass. Temporal aggregation is performed using a lightweight Transformer encoder with a single layer and four attention heads, along with learnable positional encodings. A temporal contrastive regularization term is applied during training with a fixed weight of 0.1 and a temperature of 0.5. The model uses a dropout rate of 0.7, while a rate of 0.1 is used to generate contrastive views. All regularization terms are disabled during inference.

\textbf{Hardware Configuration. }
\label{hardware}
All experiments were conducted on a high-performance computing cluster. The server was equipped with a dual-socket Intel Xeon Gold 5317 processor running at 3.00 GHz with 48 CPU cores and 755 GB of system memory. For GPU-accelerated computation, we used 8 NVIDIA RTX A6000 GPUs, each with 48 GB of GDDR6 memory, running driver version 535.274.02.

% at the University of Southern California, Information Sciences Institute

\textbf{Software Environment. }
Our implementation was developed in Python 3.10.12 on Ubuntu 22.04.5 LTS. We leveraged PyTorch 2.0.1~\cite{paszke2019pytorch} as our primary deep learning framework, compiled with CUDA 11.8~\cite{nickolls2008scalable} support for GPU acceleration. Numerical computations were performed using NumPy 1.26.4~\cite{harris2020array} and SciPy 1.15.3~\cite{virtanen2020scipy}. For machine learning utilities including preprocessing, evaluation metrics, and data splitting, we employed scikit-learn 1.7.2~\cite{pedregosa2011scikit}. Additional utilities included tqdm 4.67.1 for progress monitoring, psutil 7.1.3 for system resource tracking, and gdown 5.2.0 for dataset retrieval.

\textbf{Dataset Preprocessing. } Dynamic graph snapshots are preprocessed as follows: (1) Node features $X_t$ are standardized to zero mean and unit variance per timestep across the joint $(\text{node}\times\text{feature\_dim})$ axis using all nodes; this is a transductive node-classification setup with node-level (not time-level) splits, so val/test nodes participate in the per-timestep statistics, matching the standard protocol used by SpikeNet and prior work on these benchmarks. (2) Edge lists $E_t$ are deduplicated and sorted by node ID for efficient neighborhood lookups. (3) Isolated nodes (degree = 0 across all timesteps) are removed from the test set to avoid undefined attention computations; in addition, self-loops are added to every adjacency matrix ($A_t \leftarrow A_t \vee I$) so that any node retained in train/val with degree $0$ at time $t$ still has at least one valid neighbor for attention, eliminating undefined-softmax cases throughout training and evaluation. (4) For datasets with categorical node features (e.g., DBLP), one-hot encoding is applied before standardization.

\textbf{Model Initialization. } Weight matrices $\{\mathbf{W}_s^{(k)}, \mathbf{W}_n^{(k)}, \mathbf{W}_q^{(k)}, \mathbf{W}_k^{(k)}, \mathbf{W}_v^{(k)}\}$ are initialized using Xavier/Glorot uniform initialization. Biases are initialized to zero. LIF neuron parameters are initialized as $\tau_i = 1.0$, $V_{\text{th},i} = 1.0$ for all feature dimensions $i$. Positional encodings $\mathbf{P}_{1:T}$ are initialized randomly from $\mathcal{N}(0, 0.02)$ and subsequently optimized.

\textbf{Hyperparameter Search Strategy. }\label{app:hyperparam} We perform grid search over the following ranges for sensitivity analysis: learning rate $\eta \in \{1 \times 10^{-3}, 3 \times 10^{-3}, 5 \times 10^{-3}, 7 \times 10^{-3}, 1 \times 10^{-2}\}$, dropout $p_{\text{drop}} \in \{0.3, 0.5, 0.7, 0.9\}$, contrastive weight $\lambda \in \{0.05, 0.1, 0.2, 0.3\}$, sampling probability $p \in \{0.2, 0.4, 0.6, 0.8, 1.0\}$, and hidden dimensions $d \in \{64, 128, 256, 512\}$. The best configuration on the validation set is selected, and final test performance is reported using this configuration (single run, no test-time hyperparameter tuning). Sensitivity analysis (Section \ref{sec:sensitivity}) confirms robustness to hyperparameter choices within reasonable ranges. For the twelve baselines, we adopt the published hyperparameters reported by the original authors when available, and otherwise perform an equivalent-budget search over the corresponding learning rate/hidden dimension/dropout grid (matching the cardinality of our own search) on the validation set. This ensures comparable tuning effort across methods; we did not employ a larger budget for ChronoSpike than for any single baseline.

\subsection{Baseline Methods}
\label{app:baseline}
We benchmark ChronoSpike against 12 methods categorized into five paradigms:

\subsubsection{Static Graph Methods}
\paragraph{DeepWalk}~\cite{perozzi2014deepwalk} learns node embeddings by performing truncated random walks on the graph and treats the resulting node sequences as sentences. It optimizes a skip-gram objective to capture node co-occurrence patterns in the sampled walks.
\paragraph{Node2Vec}~\cite{grover2016node2vec} extends DeepWalk by introducing biased random walks that interpolate between breadth-first and depth-first sampling strategies. This bias enables flexible local and global exploration of network neighborhoods during training sequence generation.

Because these methods operate on static graphs, we construct a single graph by aggregating all past interactions and apply it to the final timestep.

\subsubsection{Shallow Dynamic Graph Methods}
\paragraph{HTNE}~\cite{HTNE} integrates Hawkes point processes with attention mechanisms to model the temporal formation of network neighborhoods. It learns node embeddings by capturing both temporal influence and neighborhood relevance in event sequences.
\paragraph{M$^2$DNE}~\cite{MMDNE_dblp_and_tmall} models dynamic networks using a temporal point process framework with attention mechanisms. This method jointly captures short-term interaction patterns and long-term network evolution through temporal embedding updates.
\paragraph{DyTriad}~\cite{DynamicTriad} explicitly models triadic closure processes to characterize structural evolution in dynamic networks, which lets the model effectively capture graph dynamics. It learns node embeddings by capturing motif-level temporal transitions over time.

\subsubsection{Recurrent and Neural Dynamic Methods}
\paragraph{MPNN-LSTM} (MPNN)~\cite{MPNN} combines graph message passing with recurrent neural networks to update node representations over time. Temporal dependencies are modeled through sequential hidden state propagation across graph snapshots.
\paragraph{JODIE}~\cite{kumar2019predicting_jodie} maintains dynamic embeddings for users and items that are updated at every interaction event using coupled recurrent neural networks. It further projects embeddings forward in time to support future interaction prediction.
\paragraph{EvolveGCN}~\cite{pareja2020evolvegcn} evolves the parameters of graph convolutional networks using recurrent neural networks at each timestamp. Instead of updating node embeddings directly, it updates the GCN weights to reflect graph evolution over time.

\subsubsection{Attention-Based Methods}
\paragraph{TGAT}~\cite{tgat} applies self-attention mechanisms with functional time encodings to aggregate temporal and topological information inductively. It supports inductive learning on temporal graphs and does not rely on recurrent hidden states.

\subsubsection{Spiking Graph Neural Networks}
\paragraph{SpikeNet}~\cite{spikenet23} introduces spiking neurons into graph neural networks by modeling temporal dynamics through membrane potential accumulation across snapshots. Neighborhood information is propagated via spike-based message passing over graph structures.
\paragraph{Dy-SIGN}~\cite{dysign24} formulates dynamic graph learning as an equilibrium spiking neural system, where node representations are obtained as the fixed points of an iterative spike-based message-passing process. The model is optimized using implicit differentiation, which avoids explicit backpropagation through time.
\paragraph{Delay-DSGN}~\cite{Delay_DSGN_25} incorporates learnable synaptic delay mechanisms into spiking graph message passing to account for temporal latency in signal propagation. It models temporal dependencies by associating spike transmissions with adaptive delay kernels that aggregate information across multiple time steps.

\section{Details of Experimental Results}

\subsection{Details of Ablation Study}\label{app:ablation}
We analyze ChronoSpike on the DBLP, Tmall, and Patent datasets with different training ratios to validate the contribution of each core component to its performance. We systematically removed or replaced the key modules, as shown in Table~\ref{tab:ablation_node_cl}. Removing the temporal Transformer results in the largest performance drop on DBLP and Tmall, with both Micro and Macro F1 scores decreasing by about 3--5\%, confirming that this component captures global long-range dependencies that local spiking dynamics alone cannot address. On Patent, the drop is smaller but still consistent ($\approx$2.7--3.1\%), reflecting that its shorter temporal horizon (T=25) reduces the relative advantage of global temporal attention; nonetheless, the Transformer remains beneficial across all settings. Similarly, the removal of adaptive LIF neurons consistently led to worse results on all three datasets, confirming that learnable membrane time constants are crucial to effectively model complex temporal patterns. Removing the attentive aggregator and using mean pooling caused a moderate but consistent drop across datasets, with the effect most pronounced on the dense Tmall graph. On Patent, the attention ablation yields among the lowest component-level scores at most training ratios (e.g., 0.8240 Mi-F1 at 80\%), though at 60\% the ReLU variant scores slightly lower, suggesting that both spatial attention and spike-based computation contribute meaningfully on large-scale graphs. The contrastive loss provides a consistent regularization benefit across all three datasets and training ratios.

The identity of the strongest ablation variant varies across datasets and settings. On DBLP, the ReLU variant (w/o SNN) and the adaptive LIF ablation alternate as the top-performing ablations; on Tmall, the contrastive loss ablation and the ReLU variant are frequently among the strongest. On Patent, the LSTM temporal variant performs best across all training ratios (e.g., 0.8403 Ma-F1 and 0.8391 Mi-F1 at 80\%), suggesting that Patent's large spatial scale and shorter temporal depth make recurrent temporal modeling a competitive alternative. Despite these variations, ChronoSpike consistently outperforms all ablation variants across every dataset and training ratio, with margins of approximately 2.0--2.5 percentage points over the best ablation on Patent, indicating that the combination of sparse spiking dynamics and Transformer-based temporal readout provides complementary benefits beyond what any single component achieves individually.
Finally, the static baseline confirms that effective dynamic graph learning fundamentally requires temporal modeling. The failure is most dramatic on DBLP, where performance drops to as low as 4.56\%, while on Patent, the degradation is less severe (Ma-F1 $\approx$0.805), likely because Patent's rich node features carry substantial class-discriminative information even without temporal context. Nonetheless, the consistent gap of roughly 5--6\% on Patent and the near-complete collapse on certain DBLP splits confirm that temporal modeling is indispensable.

\begin{figure}[!ht]
  \centering
  \includegraphics[width=\linewidth]{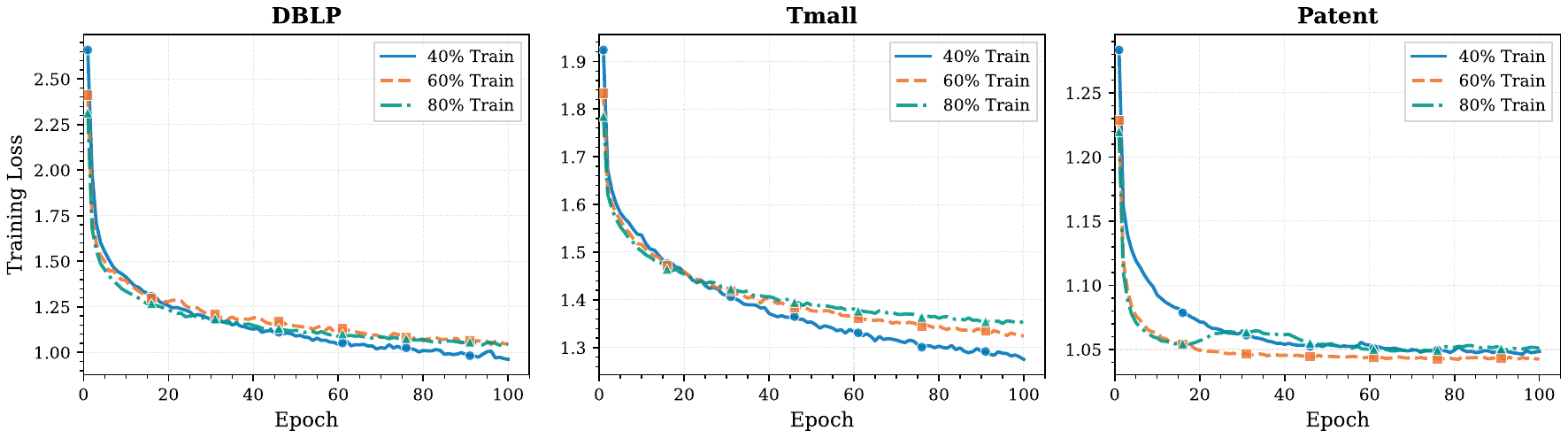}
  \caption{Training loss curves of ChronoSpike on DBLP, Tmall, and Patent with 40\%, 60\%, and 80\% training splits. ChronoSpike shows stable convergence across all settings, with faster loss reduction at higher supervision.}
  \label{fig:loss_curves}
\end{figure}

\subsection{Training Dynamics and Convergence Behavior}\label{app:train_dynamics}
Figure~\ref{fig:loss_curves} shows training loss on DBLP, Tmall, and Patent under 40\%, 60\%, and 80\% training splits. Across all datasets and splits, loss decreases smoothly, indicating stable optimization under spiking dynamics and temporal aggregation. Most loss reduction occurs within the first 20--30 epochs on DBLP and Tmall, after which curves plateau, indicating fast convergence. On DBLP and Tmall, higher training ratios yield faster convergence and lower final loss, reflecting the benefit of additional supervision for learning temporal representations. On Patent, loss decreases more slowly due to larger scale and longer temporal horizons, but remains smooth without divergence, indicating robustness over long sequences. Loss curves for different splits remain well separated, suggesting improvements arise from effective use of labeled data rather than overfitting. The absence of sharp fluctuations further indicates that surrogate-gradient training with adaptive LIF dynamics supports stable learning. ChronoSpike enables stable, scalable training on large, dynamic graphs across supervision levels.

\subsection{Interpretability Analysis}\label{app:interpretability}
We conduct an interpretability analysis on the DBLP dataset using 1000 randomly sampled test nodes to understand ChronoSpike's learned representations and mechanisms.

\subsubsection{Interpretability of Temporal Dependencies}
\label{appendix:temporal_interpretability}

To assess ChronoSpike’s ability to capture long-range dependencies, we analyze Temporal Transformer attention and compute snapshot importance by averaging attention over the full horizon ($T=27$). Figure~\ref{fig:temporal} shows that attention concentrates on early snapshots. Instead of shifting toward recent steps, as often seen in recurrent models due to vanishing gradients, ChronoSpike exhibits a primacy effect, with peak attention at $t\!=\!2$ -- $t\!=\!4$ followed by gradual decay. This indicates that the Transformer retrieves early structural context while still integrating recent updates, preserving long-range dependencies.

\begin{figure}[!ht]
  \centering
  \includegraphics[width=0.7\linewidth]{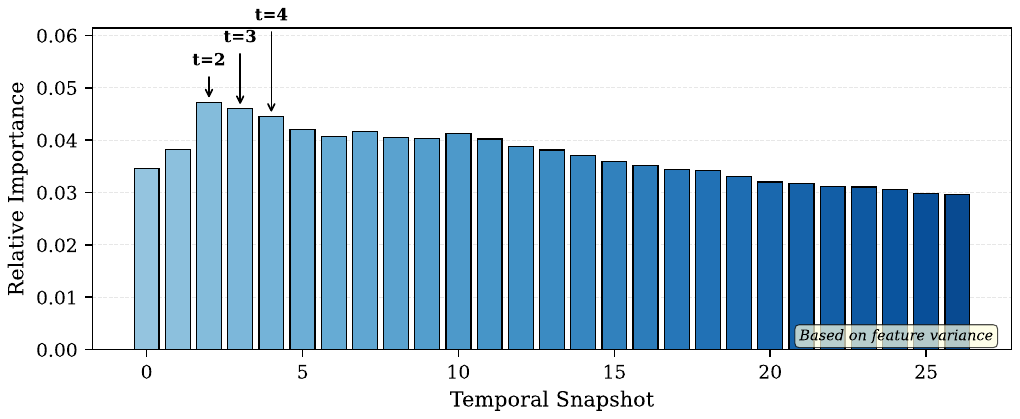}
  \caption{Learned Temporal Importance. The distribution shows a strong focus on early snapshots ($t=2, 3, 4$) rather than the most recent time steps. This shows ChronoSpike's ability to use long-range historical context.}
  \label{fig:temporal}
\end{figure}

\begin{figure}[!ht]
  \centering
  \includegraphics[width=0.7\linewidth]{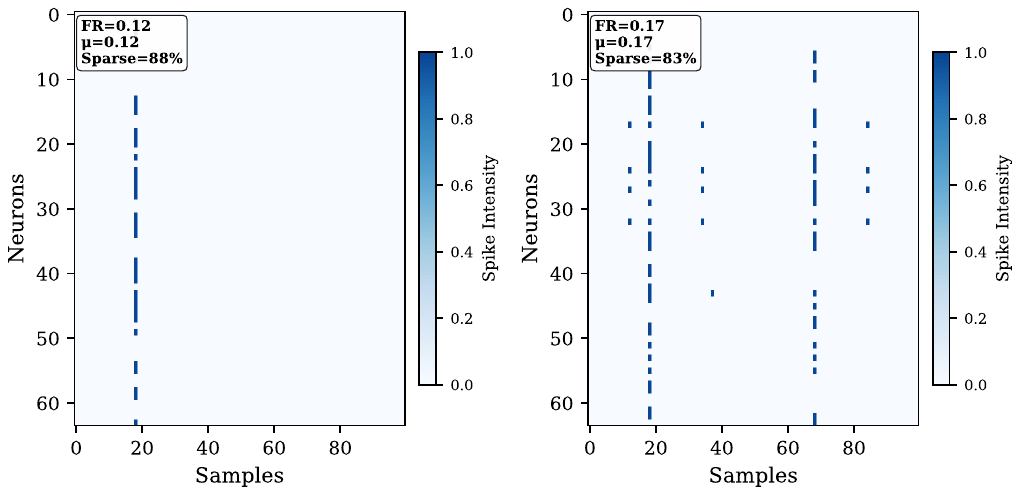}
  \caption{Spiking raster plots for two SNN layers. Spike events for Layer~1 (\textcolor{mygreen}{left}) and Layer~2 (\textcolor{myred}{right}) are shown over neurons (y-axis) and input samples (x-axis).}
  \label{fig:spikes}
\end{figure}

\begin{figure}[!ht]
  \centering
  \includegraphics[width=0.7\textwidth]{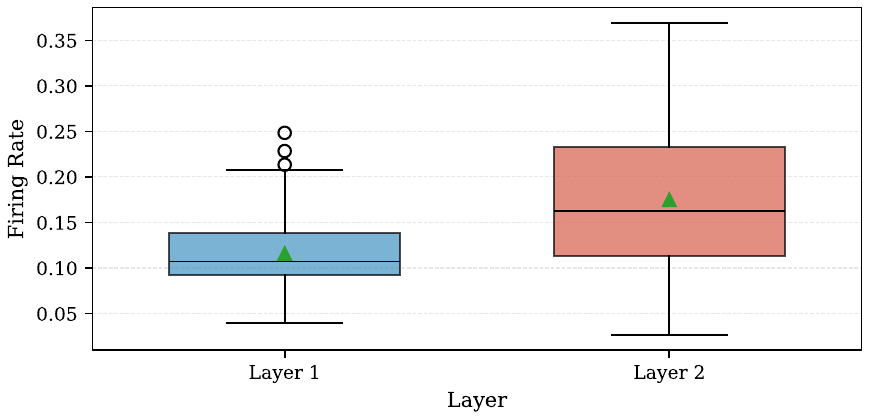}
  \caption{Distribution of per-sample firing rates across layers. Boxplots summarize firing rates (spikes per neuron per sample) aggregated over neurons for each input sample. Boxes indicate interquartile ranges with medians, whiskers extend to the most extreme non-outlier values within $1.5\times$ the IQR, and individual points beyond the whiskers represent outliers (Tukey convention); triangles indicate means. Layer~2 exhibits higher median firing rates and larger variability across samples than Layer~1.}
  \label{fig:firing_rates}
\end{figure}

\begin{figure}[!ht]
  \centering
  \includegraphics[width=0.7\linewidth]{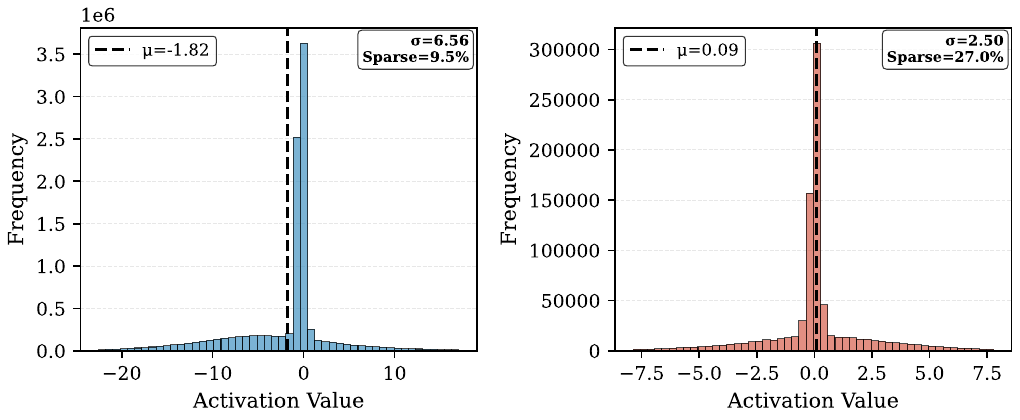}
  \caption{Histograms of membrane potentials measured before spike generation for Layer~1 (\textcolor{blue}{left}) and Layer~2 (\textcolor{myred}{right}), aggregated over all neurons, timesteps, and input samples in the evaluation set. Dashed vertical lines indicate mean values ($\mu$), and reported $\sigma$ denotes standard deviation. Here, ``\textcolor{mygreen}{Sparse}'' denotes the percentage of membrane potentials near zero $(|V|<0.01)$, indicating quiescent neural activity.}
  \label{fig:activation_dist}
\end{figure}

\subsubsection{Neuromorphic Activity Analysis}
\label{appendix:neuromorphic_activity}

We further analyze spiking behavior to verify ChronoSpike’s event-driven computation. Figure~\ref{fig:spikes} shows raster plots across neurons for a random subset of test samples, confirming highly sparse activity. Layer~1 has a silence ratio of about 88\% (firing rate $\mu \approx 0.12$), indicating selective feature detection with most neurons inactive. Layer~2 shows slightly higher activity (silence ratio $\approx 83\%$, $\mu \approx 0.17$), consistent with integrating upstream signals, while preserving sparsity.

Figure~\ref{fig:firing_rates} summarizes firing-rate distributions. Layer~1 shows a low median and narrow interquartile range, indicating stable selective encoding. Layer~2 shows a higher median and greater variability, reflecting aggregation of diverse feature combinations. Although deeper layers are more active in encoding complex patterns, the median firing rate remains below 0.20, preserving sparsity and preventing avalanche effects typical in deep SNNs, while maintaining energy efficiency.

\subsubsection{Layer Activation Statistics}
\label{appendix:activation_stats}
We examine the distribution of pre-spike membrane potentials in Figure~\ref{fig:activation_dist} to analyze internal signal propagation in ChronoSpike. The histograms show a clear shift in representations between the two spiking layers. Layer~1 operates in a high-variance regime ($\sigma=6.56$) with a negative mean ($\mu=-1.82$). This broad subthreshold distribution indicates that the first layer suppresses noise through strong inhibition while allowing only salient features to drive spiking activity. In contrast, Layer~2 exhibits a tighter distribution ($\sigma=2.50$) centered near zero ($\mu=0.09$), with a lower proportion of quiescent neurons (9.5\% subthreshold vs.\ 27.0\% in Layer~1). This shift shows that ChronoSpike progressively normalizes feature activations with depth, stabilizing dynamics, and placing deeper layers in a more active, balanced regime suited for temporal integration.

\end{document}